\documentclass[journal]{IEEEtran}
\usepackage{amsmath,amsfonts}
\usepackage{algorithmic}
\usepackage{algorithm}
\usepackage{array}
\usepackage[caption=false,font=normalsize,labelfont=sf,textfont=sf]{subfig}
\usepackage{textcomp}
\usepackage{stfloats}
\usepackage{url}
\usepackage{verbatim}
\usepackage{graphicx}
\usepackage{cite}
\usepackage{hyperref} 
\usepackage{xcolor} 
\usepackage{academicons}
\usepackage{hyperref}
\usepackage{float}
\usepackage{mathtools}
\usepackage{amssymb}
\usepackage{amsthm}
\usepackage{subcaption}
\newtheorem{assumption}{Assumption}

\newtheorem{lemma}{Lemma}

\theoremstyle{definition}

\newtheorem{remark}{Remark}

\begin{document}

\title{Modular Lie Algebraic PDE Control of Multibody Flexible Manipulators}

\author{
    S.~Yaqubi,~\href{https://orcid.org/0000-0002-7093-5747}{\includegraphics[height=2ex]{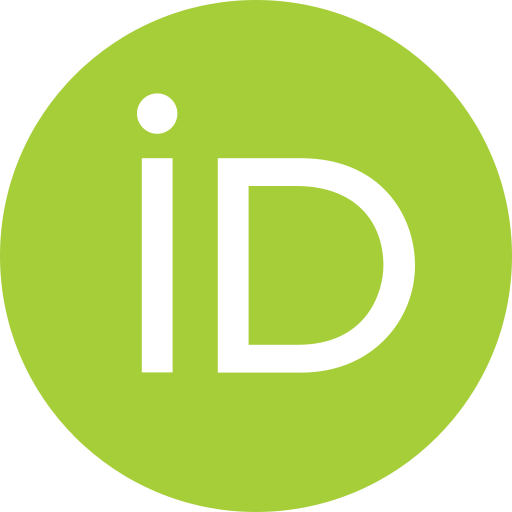}} 
    \and
     J.~Mattila\,~\href{https://orcid.org/0000-0003-1799-4323}{\includegraphics[height=2ex]{ORCID_iD.svg.jpg}}

\thanks{This work was supported by the Research Council of Finland under the project
``Nonlinear PDE-model-based control of flexible manipulators''
(Grant No.~355664). (Corresponding author: S.~Yaqubi)}
\thanks{S.~Yaqubi and J.~Mattila are with the
Department of Automation Technology and Mechanical Engineering,
Tampere University, Korkeakoulunkatu~6, 33720~Tampere, Finland
(e-mails: \href{mailto:sadeq.yaqubi@tuni.fi}{sadeq.yaqubi@tuni.fi},\href{mailto:jouni.mattila@tuni.fi}{jouni.mattila@tuni.fi}).}
\thanks{This work has been submitted to the IEEE for possible publication. Copyright may be transferred without notice , after which this version may no longer be accessible.}

}

\maketitle

\begin{abstract}
This paper presents a subsystem-based adaptive control framework
for serial flexible manipulators with an arbitrary number of links,
in which the elastic deformation PDE of each link is carried
through the entire control design without spatial discretization
or modal truncation. All dynamic quantities --- rigid-body motion,
elastic deformation, and inter-link constraint forces --- are
expressed uniformly as body-fixed twists and wrenches within the
$\mathfrak{se}(3)$ Lie-algebraic structure. A controllable form
of the per-link dynamics is derived by substituting the
strain-based deformation PDE into the dynamic equation,
eliminating distributed elastic acceleration and yielding a model
governed by the body-fixed twist acceleration and deformation
field. Desired subsystem twist trajectories are generated via a
deflection-compensating inverse kinematics procedure. A nominal
per-link controller is proven to produce exponential twist error
decay via a per-subsystem Lyapunov function $\nu_i$. An adaptive
modification replaces exact physical parameters with online
estimates governed by a projection-based law, augmenting $\nu_i$
with a parameter estimation error term. Upon summing over all links,
the interaction power terms in $\dot{\mathcal{V}} = \sum_i\dot{\nu}_i$
telescope to zero by Newton's third law and the frame invariance of
the natural power pairing on $\mathfrak{se}(3)\times\mathfrak{se}^*(3)$,
establishing exponential convergence of all twist errors and bounded
elastic deformation under both nominal and adaptive controllers.
The screw-theoretic structure renders interaction term
cancellation exact, making the stability certificate modular
and scalable to chains of arbitrary length.
The framework is validated numerically on a two-link flexible
manipulator in three-dimensional motion.
\end{abstract}

\begin{IEEEkeywords}
Screw Theory, Model-based Control, Flexible Manipulator, PDE Control, Multibody Systems.
\end{IEEEkeywords}

\section{Introduction} \label{sec:1}

Flexible multibody robotic systems present a fundamental control
challenge arising from the simultaneous presence of large rigid-body
motion on the Lie group $\mathrm{SE}(3)$ and spatially distributed
elastic deformation governed by a partial differential equation
(PDE)~\cite{Boyer2021,Yao2026}. Precise trajectory tracking is
inherently difficult: deformation interacts nonlinearly with
rigid-body motion, producing underactuation and spillover
instability when finite-dimensional approximations replace the
true deformation model~\cite{Yaqubi2023,Liu2018}. Two strategies
dominate the literature. The first replaces distributed deformation
with a truncated ODE via assumed modes or finite
elements~\cite{Walsh2015}, yielding tractable but formally
unquantifiable truncation error. The second linearizes rigid-body
kinematics~\cite{Rao2007}, restricting validity to small
displacements. The present work carries the PDE structure through
the entire control design without spatial discretization, represents
rigid-body dynamics geometrically on $\mathfrak{se}(3)$, and
establishes stability of the resulting infinite-dimensional
closed-loop system formally.

PDE-based boundary control for single flexible
links~\cite{Yang2015,Cao2017} constructs Lyapunov functionals
directly on the infinite-dimensional state, bypassing spatial
discretization and spillover entirely. However, these derivations
are strongly system-specific: the Lyapunov functional, boundary
inputs, and stability arguments must each be reconstructed for
every new configuration~\cite{Yaqubi2026}. Extension to chains of
flexible links is substantially harder --- holonomic joints
introduce constraint forces into each link's PDE boundary
conditions, and assembling a closed-loop certificate for an
$n$-link chain from per-link results has few precedents. The most
relevant contribution is Zhu~\cite{Zhu2010}, who developed a
virtual decomposition framework for multibody stability synthesis;
however, that work uses absolute rather than body-fixed
representations, forgoing Lie-algebraic structure in the dynamics
and complicating inter-link kinematics. The overwhelming majority
of PDE-based flexible manipulator controllers are designed for one
or two links at design stage~\cite{Yaqubi2023,Mohsenipour2024}; a scalable
synthesis for arbitrary $n$ remains open.

For rigid multibody systems, screw-theoretic twist and wrench
coordinates on $\mathfrak{se}(3)$ are well-established for modular
composition of link dynamics~\cite{Featherstone2008,Mller2018,Zhu2010}
and control~\cite{Koivumaki2022}, where inter-link coupling reduces
to linear Adjoint transformations. Extending this to flexible bodies
is substantially more complex: the elastic field must be lifted to a
continuous screw field~\cite{Gao2024,Cibicik2021}, and its coupling
to the rigid-body configuration causes the forward dynamic map to
take a deeply nested form that obstructs algebraic assembly and
destroys the linear operator structure on which systematic control
design depends~\cite{Li2024,Gao2023}. Representing link velocities
as body-fixed twists and internal forces as wrenches recovers this
structure: velocity-dependent forces arise from the adjoint action
on the body-fixed twist~\cite{Featherstone2008,Mller2018,Zhang2024},
the deformation field is a body-fixed continuous screw
field~\cite{Yaqubi2026,Cibicik2021}, and inter-link coupling reduces
to structured Adjoint
transformations~\cite{Featherstone2008,Zhu2010}. Screw-theoretic
model extensions --- dual screw models~\cite{Cibicik2021},
finite-element screw formulations~\cite{Grazioso2019}, Lie group
formulations for geometrically exact beams~\cite{Herrmann2024}, and
screw-theoretic synthesis models~\cite{Yaqubi20262} --- demonstrate
this potential, though systematic PDE-based control synthesis within
this structure remains open. Control of flexible multibody robots
is less developed than modeling: ODE-based methods from the floating
frame~\cite{Shabana2013,Sugiyama2006} or modal
truncation~\cite{Sharifnia2016} introduce spillover; geometrically
exact formulations~\cite{Trivedi2008,Demoures2015} and Lie group
variational integrators~\cite{Boyer2021,Herrmann2024,Chen2022}
capture large deformations but complicate control law derivation
and Lyapunov arguments.

To bridge this gap, the present work develops a subsystem-based
adaptive control architecture for flexible multibody robotic systems,
building on the screw-theoretic synthesis model
of~\cite{Yaqubi2026}. Four contributions are made. First, the
$n$-link manipulator is decomposed into per-link subsystems within
$\mathfrak{se}(3)$, with inter-link interaction wrenches appearing
explicitly at subsystem boundaries. Second, nominal per-link
controllers are proven stable via per-subsystem Lyapunov functions
$\nu_i$. Third, summing over all subsystems, the interaction terms
in $\dot{\mathcal{V}} = \sum_i\dot{\nu}_i$ cancel exactly by
Newton's third law, establishing full $n$-link stability without
cross-term analysis. Fourth, an adaptive modification replaces
uncertain parameters with online estimates; exponential convergence
of both twist error and parameter estimation error is established
through an augmented Lyapunov argument, with the inter-link
cancellation shown to hold unchanged under adaptation.
The screw-theoretic structure renders interaction term cancellation
exact, the stability proof modular, and the synthesis automatable
for arbitrary $n$. The framework is validated numerically on a
two-link flexible manipulator in three-dimensional motion.

\section{Problem Statement and Notation} \label{sec:2}

A serial flexible manipulator with $n$ links is considered, where
each link is connected to adjacent links via rotational joints and
actuated by forces and torques at its endpoints. The single-link
configuration and multi-link decomposition are illustrated in
Fig.~\ref{fig:1}.

\begin{figure*}[htbp]
\centering
\subfloat[]{
\includegraphics[width=0.48\textwidth]{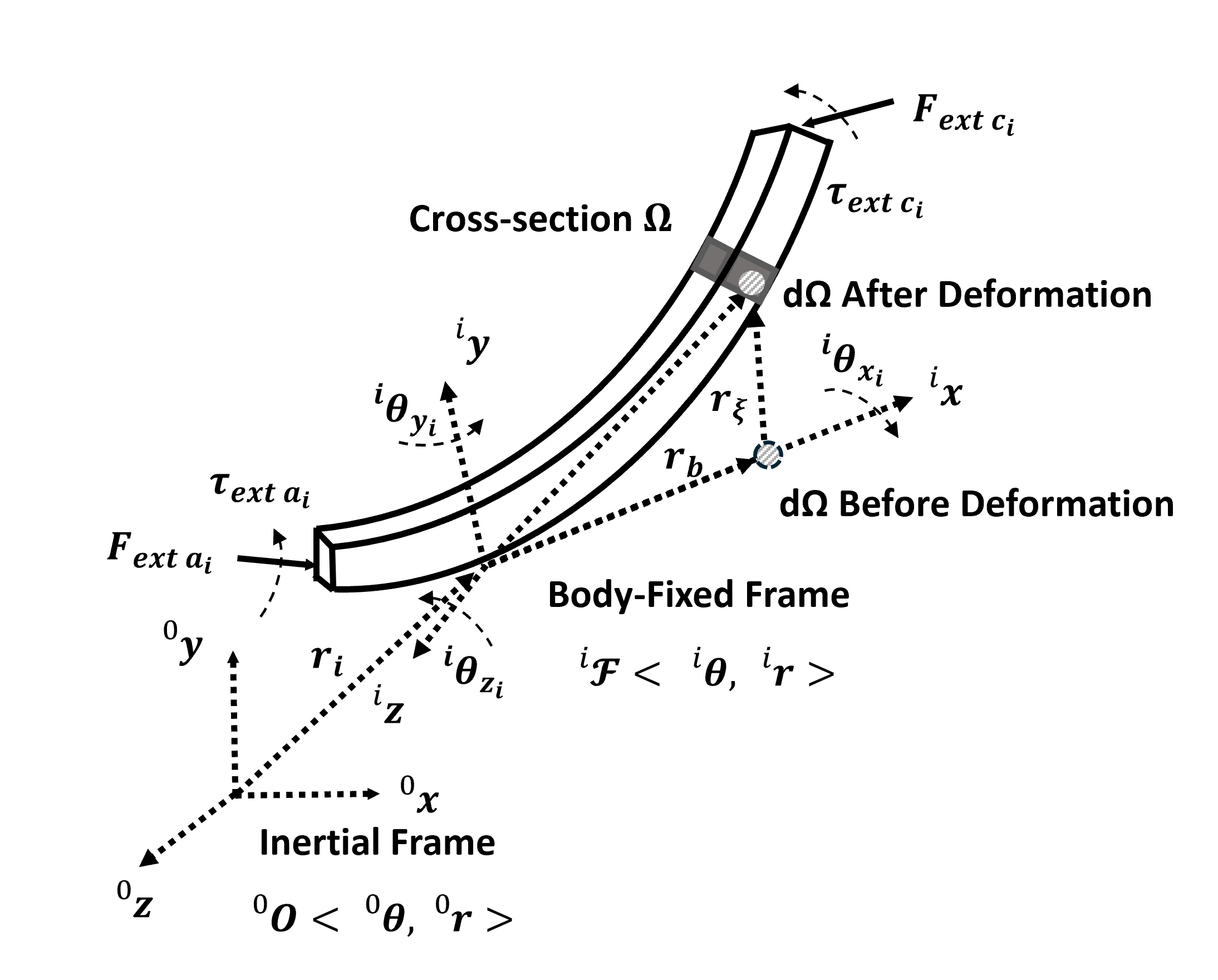}\label{fig:1a}}
\hfill
\subfloat[]{
\includegraphics[width=0.48\textwidth]{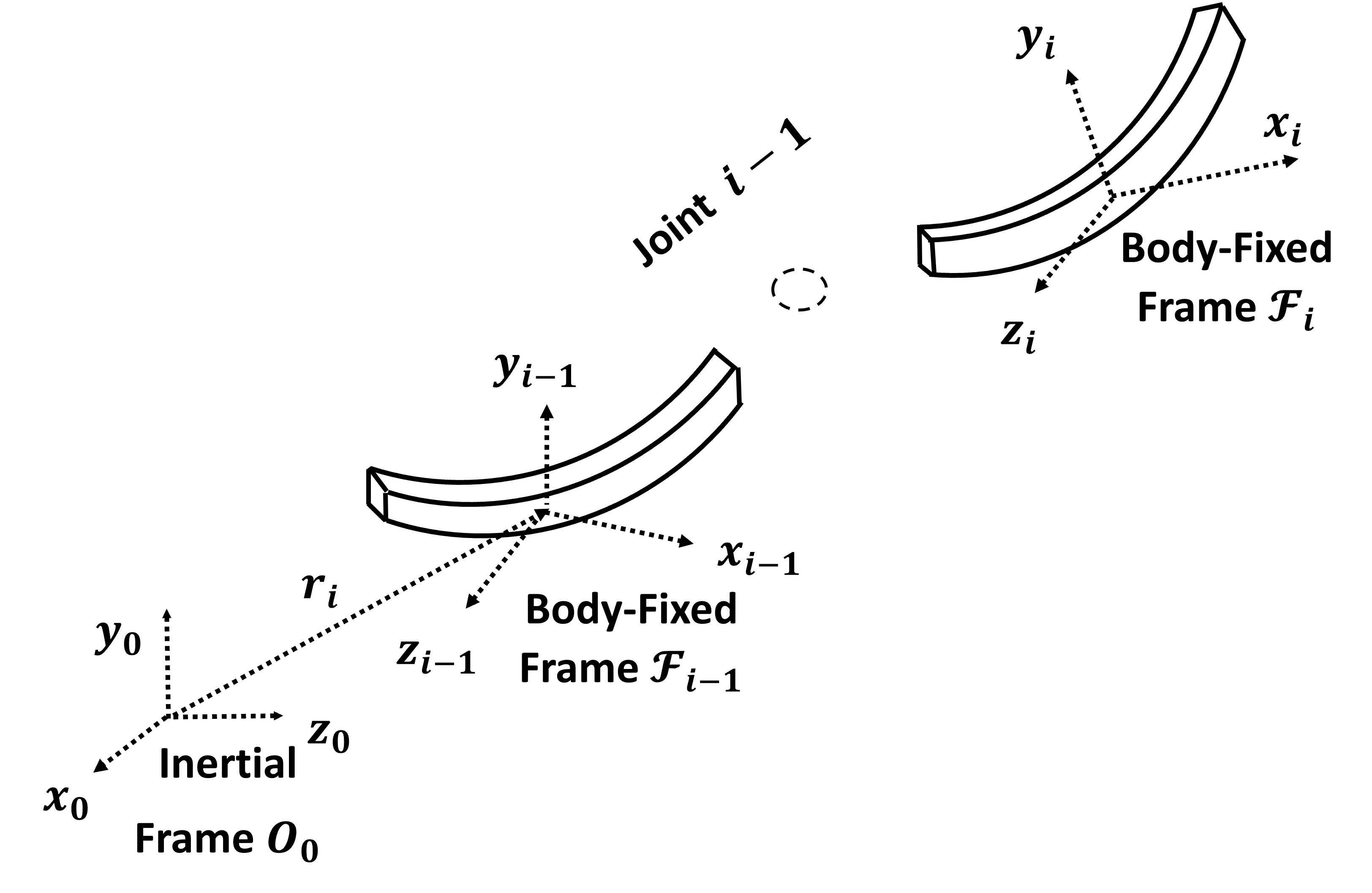}\label{fig:1b}}
\caption{(a)~Dynamic motion of flexible body with respect to
inertial and body-fixed frames, with external forces and torques
shown in coordinate form. (b)~Decomposed consecutive links
with holonomic constraints at rotational joints, with inter-link
interaction wrenches $\boldsymbol{\mathcal{W}}_{J,i}$ shown in
the Lie-algebraic $\mathfrak{se}^*(3)$ representation.}
\label{fig:1}
\end{figure*}

Each flexible link $i$ is indexed by the spatial coordinate
$\xi_i \in [a_i, c_i]$ along the beam axis in the body-fixed
frame ${}^i\!\mathbf{\mathcal{F}}$, with cross-sectional area
$A_i$, mass density $\rho_i$, elastic modulus $E_i$, and beam
length $l_i = c_i - a_i$.

The generalized coordinate vector, velocity, and body-fixed twist
are

\begin{align}
{}^{i}\mathbf{q}_i &=
\begin{bmatrix}
{}^{i}\boldsymbol{\theta}_i \\
{}^{i}\mathbf{r}_i
\end{bmatrix}, \quad
{}^{i}\dot{\mathbf{q}}_i =
\begin{bmatrix}
{}^{i}\dot{\boldsymbol{\theta}}_i \\
{}^{i}\dot{\mathbf{r}}_i
\end{bmatrix},
\label{eq:1} \\
{}^{i}\mathbf{V}_i &=
\begin{bmatrix}
{}^{i}\boldsymbol{\omega}_i \\
{}^{i}\mathbf{v}_i
\end{bmatrix}
=
\begin{bmatrix}
{}^{i}\boldsymbol{\omega}_i \\
{}^{i}\dot{\mathbf{r}}_i + {}^{i}\tilde{\boldsymbol{\omega}}_i\,{}^{i}\mathbf{r}_i
\end{bmatrix}
\in \mathfrak{se}(3),
\label{eq:twist}
\end{align}

where ${}^{i}\mathbf{r}_i \in \mathbb{R}^3$ is the body-fixed
origin position and ${}^{i}\boldsymbol{\theta}_i \in \mathbb{R}^3$
parametrizes $SO(3)$ via the exponential map
$\mathbf{R}_{oi} = \exp({}^{i}\tilde{\boldsymbol{\theta}}_i)$,
with $\tilde{(\cdot)} \in \mathfrak{so}(3)$ denoting the
skew-symmetric matrix of a vector~\cite{murray1994}. The angular
velocity relates to the coordinate derivative via
$\boldsymbol{\omega}_i =
\mathbf{J}(\boldsymbol{\theta}_i)\dot{\boldsymbol{\theta}}_i$,
where $\mathbf{J}$ is the body-frame Jacobian of the exponential
map~\cite{murray1994}. The rotation matrix $\mathbf{R}_{oi}$
transforms vectors from $\mathbf{\mathcal{F}}_i$ to the inertial
frame, chaining as $\mathbf{R}_{oi} =
\mathbf{R}_{o(i-1)}\mathbf{R}_{(i-1)i}$. The twist transforms
via the Adjoint representation ${}^{0}\mathbf{V}_i =
\mathbf{Ad}_{oi}\,{}^{i}\mathbf{V}_i$, where for any two
frames $a$ and $b$~\cite{murray1994}

\begin{align}
\mathbf{Ad}_{ba} =
\begin{bmatrix}
\mathbf{R}_{ba} & \mathbf{0} \\
\widetilde{{}^{b}\mathbf{r}_a}\,\mathbf{R}_{ba} & \mathbf{R}_{ba}
\end{bmatrix} \in \mathbb{R}^{6\times6}.
\label{eq:Ad}
\end{align}

The small adjoint $\mathrm{ad}_{\mathbf{V}_i} \in
\mathbb{R}^{6\times6}$ is the matrix representation of the Lie
bracket on $\mathfrak{se}(3)$~\cite{murray1994}

\begin{align}
\mathrm{ad}_{\mathbf{V}_i} =
\begin{bmatrix}
\tilde{\boldsymbol{\omega}}_i & \mathbf{0} \\
\tilde{\mathbf{v}}_i & \tilde{\boldsymbol{\omega}}_i
\end{bmatrix}.
\label{eq:25}
\end{align}

The wrenches $\boldsymbol{\mathcal{W}}_{ext a_i},
\boldsymbol{\mathcal{W}}_{ext c_i} \in \mathfrak{se}^*(3)$
applied at the base and tip of link $i$ pair with the twist via
the power product
$\boldsymbol{\mathcal{W}}^\top\mathbf{V} =
\boldsymbol{\tau}^\top\boldsymbol{\omega} +
\mathbf{F}^\top\mathbf{v}$,
representing instantaneous delivered power. Interaction wrenches
$\boldsymbol{\mathcal{W}}_{Ji} \in \mathfrak{se}^*(3)$ represent
constraint forces between adjacent links through the joint.

\begin{assumption}
The holonomic joint constraint forces
$\boldsymbol{\mathcal{W}}_{J,i}$ remain in the admissible
set determined by the joint kinematics and the applied
wrenches $\boldsymbol{\mathcal{W}}_i$ throughout the
closed-loop trajectory. Specifically, the constraint
solver~\cite{Yaqubi20262} guarantees that
$\|\boldsymbol{\mathcal{W}}_{J,i}\| \leq
\bar{W}_J < \infty$ whenever $\mathbf{V}_i$ and
$\dot{\mathbf{V}}_i$ are bounded.
\label{ass:WJ}
\end{assumption}

Under Euler--Bernoulli beam theory (planar cross-sections,
small deformation, isotropic material, negligible shear), the
reference coordinates ${}^{i}\mathbf{r}_{b_i}$ and elastic
displacement ${}^{i}\mathbf{r}_{\xi_i}$ satisfy

\begin{align}
{}^{i}\mathbf{r}_{b_i} &=
\begin{bmatrix} \xi_i \\ {}^{i}r_{y_{b_i}} \\ {}^{i}r_{z_{b_i}}
\end{bmatrix}, \quad
{}^{i}\dot{\mathbf{q}}_{b_i} = \mathbf{0},
\label{eq:5}
\end{align}

where zero rotational components reflect the Euler--Bernoulli
assumption of no independent cross-section rotation, and
${}^{i}\dot{\mathbf{q}}_{b_i} = \mathbf{0}$ confirms the
reference geometry is time-invariant in $\mathbf{\mathcal{F}}_i$.
The element position relative to body-fixed and inertial frames
follows as $\mathbf{r}_{ib_i} = \mathbf{r}_{\xi_i} +
\mathbf{r}_{b_i}$ and $\mathbf{r}_{ob_i} = \mathbf{r}_i +
\mathbf{r}_{ib_i}$. The body-fixed twists $\mathbf{V}_{\xi_i}$
and $\mathbf{V}_{b_i}$ are derived from the transport theorem
in~\cite{Yaqubi2026}.

Throughout, superscript $i$ is suppressed for quantities
expressed in $\mathbf{\mathcal{F}}_i$: $\mathbf{r}_i =
{}^{i}\mathbf{r}_i$, $\mathbf{r}_{\xi_i} =
{}^{i}\mathbf{r}_{\xi_i}$, $\mathbf{V}_i =
{}^{i}\mathbf{V}_i$, and similarly for all body-fixed
coordinates, derivatives, and twists.

\section{Screw-Theoretic PDE Dynamic Model} \label{sec:3}

The dynamic model of flexible link $i$ was derived
in~\cite{Yaqubi2026} via variational analysis of the link
Hamiltonian on $SE(3)$, yielding a symmetric positive-definite
inertia operator and Coriolis terms from the adjoint action
$\mathrm{ad}_{\mathbf{V}_i}$. The deformation field
$\mathbf{r}_{\xi_i}(\xi_i,t)$ is governed by a PDE preserving
the infinite-dimensional flexibility dynamics, enabling
extension to multibody chains via chained link dynamics and
constraint-based interaction wrenches, with the multibody
synthesis and inter-link wrench derivation developed
in~\cite{Yaqubi20262}. The model comprises
the dynamic equation~(\ref*{eq:26}), deformation
PDE~(\ref*{eq:34}), and boundary
conditions~(\ref*{eq:35})--(\ref*{eq:38}).

The dynamic equation is

\begin{align}
\mathbf{M}_i\,\dot{\mathbf{V}}_i +
\mathbf{D}_i(\dot{\mathbf{V}}_{\xi_i})
+ \mathbf{H}_i + \boldsymbol{\mathcal{W}}_{Ji}
= \boldsymbol{\mathcal{W}}_i,
\label{eq:26}
\end{align}

where $\dot{\mathbf{V}}_i =
[\dot{\boldsymbol{\omega}}_i^\top,\,
(\ddot{\mathbf{r}}_i - \tilde{\mathbf{r}}_i\dot{\boldsymbol{\omega}}_i
+ \tilde{\boldsymbol{\omega}}_i\dot{\mathbf{r}}_i)^\top]^\top$
is the body-fixed twist derivative~\cite{Yaqubi2026}, and
$\dot{\mathbf{V}}_{\xi_i}$ is its elastic counterpart. The
model terms are

\begin{gather}
\mathbf{M}_i =
\begin{bmatrix}
\mathbf{I}_{b_i} &
\rho_i A_i\int_{a_i}^{c_i}\tilde{\mathbf{r}}_{b_i}\,d\xi \\[4pt]
-\rho_i A_i\int_{a_i}^{c_i}\tilde{\mathbf{r}}_{b_i}\,d\xi &
m_i \mathbf{I}_3
\end{bmatrix}, \label{eq:27} \\[6pt]
\mathbf{D}_i(\dot{\mathbf{V}}_{\xi_i}) = \rho_i A_i
\begin{bmatrix}
\int_{a_i}^{c_i}\tilde{\mathbf{r}}_{ib_i}\dot{\mathbf{v}}_{\xi_i}\,d\xi \\[4pt]
\int_{a_i}^{c_i}\dot{\mathbf{v}}_{\xi_i}\,d\xi
\end{bmatrix}, \label{eq:28} \\[6pt]
\mathbf{H}_i =
\begin{bmatrix}
\rho_i A_i\int_{a_i}^{c_i}
(\tilde{\mathbf{r}}_{ib_i}\tilde{\boldsymbol{\omega}}_i\mathbf{v}_{ib_i}
+ \tilde{\mathbf{r}}_{ib_i}\mathbf{R}_{oi}^\top\mathbf{g})\,d\xi \\[4pt]
\rho_i A_i\int_{a_i}^{c_i}\tilde{\boldsymbol{\omega}}_i
\mathbf{v}_{ib_i}\,d\xi + m_i\mathbf{R}_{oi}^\top\mathbf{g}
\end{bmatrix}, \label{eq:29}
\end{gather}

\begin{align}
\boldsymbol{\mathcal{W}}_i =
\begin{bmatrix}
\boldsymbol{\tau}_{extB_i} - \boldsymbol{\tau}_{extT_i}
- \tilde{\mathbf{r}}_{\xi(a_i)}\mathbf{F}_{extB_i}
+ \tilde{\mathbf{r}}_{\xi(c_i)}\mathbf{F}_{extT_i} \\[4pt]
\mathbf{F}_{extB_i} - \mathbf{F}_{extT_i}
\end{bmatrix}, \label{eq:30}
\end{align}

where $\mathbf{v}_{ib_i} = \mathbf{v}_{\xi_i} + \mathbf{v}_{b_i}$,
$m_i = \rho_i A_i l_i$, and $\mathbf{g}$ is gravitational
acceleration. The moment of inertia is

\begin{align}
\mathbf{I}_{b_i} = -\rho_i A_i
\int_{a_i}^{c_i}\tilde{\mathbf{r}}_{b_i}^2\,d\xi,
\label{eq:Iibi}
\end{align}

and the cross-section stiffness matrices are

\begin{align}
\mathbf{I}_{v1_i} &=
\mathrm{diag}(E_iA_i,\, 0,\, 0),
\label{eq:32} \\
\mathbf{I}_{v2_i} &=
\mathrm{diag}(0,\, E_iI_{z_i},\, E_iI_{y_i}),
\label{eq:33}
\end{align}

where $I_{y_i}$, $I_{z_i}$ are the second moments of area.

The deformation PDE is

\begin{equation}
\begin{aligned}
\dot{\mathbf{v}}_i &+ \dot{\mathbf{v}}_{\xi_i}
- \tilde{\mathbf{r}}_{b_i}\dot{\boldsymbol{\omega}}_i
+ \tilde{\boldsymbol{\omega}}_i\mathbf{v}_{ib_i}
+ \mathbf{R}_{oi}^\top\mathbf{g} \\
&+ \frac{1}{\rho_i A_i}\mathbf{I}_{v2_i}\mathbf{r}_{\xi_i}''''
- \frac{1}{\rho_i A_i}\mathbf{I}_{v1_i}\mathbf{r}_{\xi_i}'' \\
=\; &\boldsymbol{\mathcal{W}}_{ia}\delta_i(a_i)
- \boldsymbol{\mathcal{W}}_{ic}\delta_i(c_i)
+ \boldsymbol{\mathcal{W}}_{J_ia}\delta_i(a_i)
- \boldsymbol{\mathcal{W}}_{J_ic}\delta_i(c_i),
\end{aligned}
\label{eq:34}
\end{equation}

where $\delta_i$ is the Dirac delta at the respective endpoints.
Integrating~(\ref*{eq:34}) over $\xi_i \in [a_i,c_i]$ and
subtracting the translational row of~(\ref*{eq:26}) divided by
$\rho_iA_i$, the terms $\dot{\mathbf{v}}_i$,
$\mathbf{R}_{oi}^\top\mathbf{g}$,
$\tilde{\mathbf{r}}_{b_i}\dot{\boldsymbol{\omega}}_i$, and
$\tilde{\boldsymbol{\omega}}_i\mathbf{v}_{b_i}$ cancel
identically --- the last via $\mathbf{v}_{ib_i} =
\mathbf{v}_{\xi_i} + \mathbf{v}_{b_i}$ --- and the boundary
wrench terms cancel by the natural boundary
conditions~(\ref*{eq:35})--(\ref*{eq:38}), leaving the
strain-rate PDE

\begin{align}
\dot{\mathbf{v}}_{\xi_i}
+ \tilde{\boldsymbol{\omega}}_i\mathbf{v}_{\xi_i}
+ \frac{1}{\rho_i A_i}\mathbf{I}_{v2_i}\mathbf{r}_{\xi_i}''''
- \frac{1}{\rho_i A_i}\mathbf{I}_{v1_i}\mathbf{r}_{\xi_i}'' = 0.
\label{eq:strainpde}
\end{align}

The boundary conditions are

\begin{align}
\mathbf{I}_{v2_i}\mathbf{r}_{\xi_i}'''(a_i)
- \mathbf{I}_{v1_i}\mathbf{r}_{\xi_i}'(a_i)
&= \mathbf{F}_{extB_i}, \label{eq:35} \\
\mathbf{I}_{v2_i}\mathbf{r}_{\xi_i}'''(c_i)
- \mathbf{I}_{v1_i}\mathbf{r}_{\xi_i}'(c_i)
&= \mathbf{F}_{extT_i}, \label{eq:36} \\
-\mathbf{I}_{v2_i}\mathbf{r}_{\xi_i}''(a_i)
&= \mathbf{H}\boldsymbol{\tau}_{extB_i}, \label{eq:37} \\
-\mathbf{I}_{v2_i}\mathbf{r}_{\xi_i}''(c_i)
&= \mathbf{H}\boldsymbol{\tau}_{extT_i}, \label{eq:38}
\end{align}

where

\begin{align}
\mathbf{H} =
\begin{bmatrix}
0 & 0 & 0 \\ 0 & 0 & 1 \\ 0 & 1 & 0
\end{bmatrix}.
\label{eq:39}
\end{align}

The terms $\mathbf{I}_{v1_i}\mathbf{r}_{\xi_i}''$ and
$\mathbf{I}_{v2_i}\mathbf{r}_{\xi_i}''''$ are the elastic
restoring forces per unit length from axial and bending
stiffness. The boundary
conditions~(\ref*{eq:35})--(\ref*{eq:38}) equate internal
elastic forces and moments at the endpoints to the externally
applied wrenches. The matrix $\mathbf{H}$ selects the bending moment
components from the applied torque wrench, reflecting
the Euler--Bernoulli assumption that curvature couples
only to the transverse torque components $\tau_y$ and
$\tau_z$ at each endpoint. 

Equations~(\ref*{eq:26}) and~(\ref*{eq:34})
constitute a coupled PDE system:~(\ref*{eq:26}) is an ODE governing
body-fixed frame motion, while~(\ref*{eq:34}) is a fourth-order PDE in
$\xi_i$ and second-order in $t$ governing the distributed
deformation field $\mathbf{r}_{\xi_i}(\xi_i,t)$, coupled
through $\mathbf{D}_i$ which introduces elastic acceleration
into the rigid-body equation. This system is well-posed for individual links and the
multibody synthesis~\cite{Yaqubi2026,Yaqubi20262}, and
forms the basis for the subsystem-based controller of
Section~\ref{sec:4}.

\section{Subsystem-based controller for flexible link}
\label{sec:4}

The model~(\ref*{eq:26}) with coupled deformation
PDE~(\ref*{eq:34}) is not directly suitable for feedback
control: the distributed inertia term
$\mathbf{D}_i(\dot{\mathbf{V}}_{\xi_i})$ involves
acceleration-level distributed states uncontrollable with
the limited actuation wrenches $\boldsymbol{\mathcal{W}}_i$.
Substituting the strain-rate PDE~(\ref*{eq:strainpde})
into~(\ref*{eq:26}) eliminates $\mathbf{D}_i$ and yields
a controllable form affine in $\dot{\mathbf{V}}_i$,
developed in Section~\ref{subsec:4b}. The controller is
first designed for the nominal system; adaptive
modification follows in Section~\ref{sec:5}.

\subsection{Reference twist generation} \label{subsec:4a}

The actual inertial tip position of the terminal link is

\begin{align}
{}^0\mathbf{p}_e
= \mathbf{f}_{\mathrm{kin}}(\mathbf{q}_j)
+ \boldsymbol{\delta}_e, \qquad
\boldsymbol{\delta}_e
= \sum_{i=1}^{n} \mathbf{R}_{oi}\,\mathbf{r}_{\xi_i}(c_i,t),
\label{eq:pe}
\end{align}

where $\mathbf{f}_{\mathrm{kin}}$ is the rigid-body
forward-kinematics map and $\boldsymbol{\delta}_e$ is the
cumulative tip deflection in the inertial frame. Because
$\boldsymbol{\delta}_e$ depends on the distributed deformation
field governed by~(\ref*{eq:34}), the joint-space reference
must be corrected at every sample. Given the desired endpoint position ${}^0\mathbf{p}_{e,d}(t)$ and the current deflection
estimate

\begin{align}
\boldsymbol{\delta}_e^{\mathrm{est}}(t)
= \sum_{i=1}^{n} \mathbf{R}_{oi}\,\mathbf{r}_{\xi_i}^{\mathrm{meas}}(c_i,t),
\label{eq:delta_est}
\end{align}

where $\mathbf{r}_{\xi_i}^{\mathrm{meas}}(c_i,t)$ is the tip
deformation obtained from onboard sensing (strain gauges or
tip-mounted accelerometers), which approximates the true field
$\mathbf{r}_{\xi_i}(c_i,t)$ to within sensor accuracy, the deflection-corrected
rigid-body target and desired joint coordinates are

\begin{align}
{}^0\mathbf{p}_{e,d}^{\mathrm{corr}}
&= {}^0\mathbf{p}_{e,d} - \boldsymbol{\delta}_e^{\mathrm{est}},
\label{eq:p_corr} \\
\mathbf{q}_{j,d}
&= \mathbf{f}_{\mathrm{kin}}^{-1}\!
   \left({}^0\mathbf{p}_{e,d}^{\mathrm{corr}}\right).
\label{eq:qjd}
\end{align}

At the velocity level, differentiating~(\ref*{eq:pe}) and
solving for the desired joint velocity gives

\begin{align}
\dot{\mathbf{q}}_{j,d}
= \mathbf{J}_{\mathrm{kin}}^{\dagger}(\mathbf{q}_{j,d})
  \!\left({}^0\dot{\mathbf{p}}_{e,d}
         - \dot{\boldsymbol{\delta}}_e^{\mathrm{est}}\right),
\label{eq:qdotjd}
\end{align}

where $\mathbf{J}_{\mathrm{kin}} =
\partial\mathbf{f}_{\mathrm{kin}}/\partial\mathbf{q}_j$,
$(\cdot)^\dagger$ is the Moore--Penrose pseudoinverse, and

\begin{align}
\dot{\boldsymbol{\delta}}_e^{\mathrm{est}}
= \sum_{i=1}^{n}\!\left(
  \tilde{\boldsymbol{\omega}}_i\mathbf{R}_{oi}
  \mathbf{r}_{\xi_i}(c_i,t)
  + \mathbf{R}_{oi}\dot{\mathbf{r}}_{\xi_i}(c_i,t)
  \right).
\label{eq:deltadot}
\end{align}

The Jacobian is evaluated at $\mathbf{q}_{j,d}$ to keep the
reference velocity consistent with the corrected
target~(\ref*{eq:qjd}). The desired body-fixed twist for
link $i$ follows from~(\ref*{eq:twist}) evaluated along the
desired trajectory

\begin{align}
\mathbf{V}_{d,i}
= \begin{bmatrix}
    \mathbf{J}(\boldsymbol{\theta}_{d,i})
    \dot{\boldsymbol{\theta}}_{j,d,i} \\[3pt]
    \dot{\mathbf{r}}_{d,i}
    + \widetilde{\mathbf{J}(\boldsymbol{\theta}_{d,i})
      \dot{\boldsymbol{\theta}}_{j,d,i}}\,\mathbf{r}_{d,i}
  \end{bmatrix}
\in \mathfrak{se}(3).
\label{eq:Vdi}
\end{align}

The reference generation operates independently of the
subsystem controllers: the deflection
correction~(\ref*{eq:p_corr}) modifies only the joint-space
reference, while each subsystem controller tracks its own
$\mathbf{V}_{d,i}$ via the twist error
$\mathbf{e}_{\mathbf{V}_i} = \mathbf{V}_{d,i} - \mathbf{V}_i$.
The deformation field is not commanded directly but regulated
indirectly through the corrected joint reference and the
elastic boundary conditions~(\ref*{eq:35})--(\ref*{eq:38}),
keeping the stability analysis of
Section~\ref{subsec:4c} unaffected by reference generation.

\subsection{Dynamic model in controllable form} \label{subsec:4b}

To obtain a controllable form of~(\ref*{eq:26}), the distributed
inertia term $\mathbf{D}_i(\dot{\mathbf{V}}_{\xi_i})$ is eliminated
by substituting the strain-rate PDE~(\ref*{eq:strainpde}), which
expresses $\dot{\mathbf{v}}_{\xi_i}$ algebraically in terms of the
deformation field and the rigid-body motion as

\begin{align}
\dot{\mathbf{v}}_{\xi_i}
= -\tilde{\boldsymbol{\omega}}_i\mathbf{v}_{\xi_i}
  - \frac{1}{\rho_i A_i}\mathbf{I}_{v2_i}\mathbf{r}_{\xi_i}''''
  + \frac{1}{\rho_i A_i}\mathbf{I}_{v1_i}\mathbf{r}_{\xi_i}''.
\label{eq:vxidot_sub}
\end{align}

Substituting~(\ref*{eq:vxidot_sub}) into the distributed inertia
operator~(\ref*{eq:28}) yields

\begin{equation}
\begin{aligned}
\mathbf{D}_i(\dot{\mathbf{V}}_{\xi_i})
&= \rho_i A_i
\begin{bmatrix}
\displaystyle\int_{a_i}^{c_i} \tilde{\mathbf{r}}_{ib_i}\dot{\mathbf{v}}_{\xi_i}\,d\xi \\[8pt]
\displaystyle\int_{a_i}^{c_i} \dot{\mathbf{v}}_{\xi_i}\,d\xi
\end{bmatrix} \\
&= \begin{bmatrix}
  \begin{aligned}
    &-\rho_i A_i\int_{a_i}^{c_i} \tilde{\mathbf{r}}_{ib_i}\tilde{\boldsymbol{\omega}}_i \mathbf{v}_{\xi_i}\,d\xi \\
    &\quad +\int_{a_i}^{c_i} \tilde{\mathbf{r}}_{ib_i} \!\left(\mathbf{I}_{v1_i}\mathbf{r}_{\xi_i}'' - \mathbf{I}_{v2_i}\mathbf{r}_{\xi_i}''''\right)d\xi 
  \end{aligned} \\[18pt]
  \begin{aligned}
    &-\rho_i A_i\int_{a_i}^{c_i} \tilde{\boldsymbol{\omega}}_i\mathbf{v}_{\xi_i}\,d\xi \\
    &\quad +\int_{a_i}^{c_i} \!\left(\mathbf{I}_{v1_i}\mathbf{r}_{\xi_i}'' - \mathbf{I}_{v2_i}\mathbf{r}_{\xi_i}''''\right)d\xi
  \end{aligned}
\end{bmatrix}.
\end{aligned}
\label{eq:Di_sub}
\end{equation}

Adding~(\ref*{eq:Di_sub}) to the generalized force
term~(\ref*{eq:29}), the velocity-dependent Coriolis integrals
combine via $\mathbf{v}_{ib_i} = \mathbf{v}_{\xi_i} + \mathbf{v}_{b_i}$.

\begin{equation}
\begin{aligned}
&\rho_i A_i\int_{a_i}^{c_i}
  \tilde{\boldsymbol{\omega}}_i\mathbf{v}_{ib_i}\,d\xi
- \rho_i A_i\int_{a_i}^{c_i}
  \tilde{\boldsymbol{\omega}}_i\mathbf{v}_{\xi_i}\,d\xi \\
&\quad= \rho_i A_i\int_{a_i}^{c_i}
  \tilde{\boldsymbol{\omega}}_i\mathbf{v}_{b_i}\,d\xi,
\end{aligned}
\label{eq:coriolis_cancel}
\end{equation}

and identically for the rotational block. The $\mathbf{v}_{\xi_i}$
dependence therefore cancels exactly between $\mathbf{H}_i$
and $\mathbf{D}_i(\dot{\mathbf{V}}_{\xi_i})$, and substituting
into~(\ref*{eq:26}) yields the controllable form

\begin{align}
\mathbf{M}_i\,\dot{\mathbf{V}}_i
+ \mathbf{H}_{c,i}
+ \boldsymbol{\mathcal{W}}_{J,i}
= \boldsymbol{\mathcal{W}}_i,
\label{eq:contform}
\end{align}

where $\mathbf{H}_{c,i}$ collects all remaining
velocity- and deformation-dependent terms

\begin{equation}
\mathbf{H}_{c,i} = \begin{bmatrix}
\begin{aligned}
    & \rho_i A_i\int_{a_i}^{c_i} \tilde{\mathbf{r}}_{ib_i}
      \tilde{\boldsymbol{\omega}}_i\mathbf{v}_{b_i}\,d\xi \\
    & + \rho_i A_i\int_{a_i}^{c_i}\tilde{\mathbf{r}}_{ib_i}
      \mathbf{R}_{oi}^\top\mathbf{g}\,d\xi \\
    & + \int_{a_i}^{c_i} \tilde{\mathbf{r}}_{ib_i}
      (\mathbf{I}_{v1_i}\mathbf{r}_{\xi_i}''
      - \mathbf{I}_{v2_i}\mathbf{r}_{\xi_i}'''')\,d\xi
\end{aligned} \\[28pt]
\begin{aligned}
    & \rho_i A_i\int_{a_i}^{c_i}
      \tilde{\boldsymbol{\omega}}_i\mathbf{v}_{b_i}\,d\xi
      + m_i\mathbf{R}_{oi}^\top\mathbf{g} \\
    & + \int_{a_i}^{c_i}
      (\mathbf{I}_{v1_i}\mathbf{r}_{\xi_i}''
      - \mathbf{I}_{v2_i}\mathbf{r}_{\xi_i}'''')\,d\xi
\end{aligned}
\end{bmatrix}.
\label{eq:Hci}
\end{equation}

The controllable form~(\ref*{eq:contform}) is affine in
$\dot{\mathbf{V}}_i$: $\mathbf{H}_{c,i}$ depends only on
$\mathbf{V}_i$ and $\mathbf{r}_{\xi_i}$, with the distributed
elastic acceleration $\dot{\mathbf{v}}_{\xi_i}$ fully absorbed
into the stiffness integrals via~(\ref*{eq:vxidot_sub}),
making~(\ref*{eq:contform}) amenable to model-based feedback
design.

\subsection{Subsystem-based controller design and stability proof}
\label{subsec:4c}

The subsystem-level twist error is defined as

\begin{align}
\mathbf{e}_{\mathbf{V}_i} =  \mathbf{V}_{d,i} - \mathbf{V}_i\label{eq:twisterror}
\end{align}

\begin{remark}
The twist error $\mathbf{e}_{\mathbf{V}_i} =
\mathbf{V}_{d,i} - \mathbf{V}_i$ is formed as a difference
of coordinate vectors in $\mathbb{R}^6$, which is admissible
because both $\mathbf{V}_{d,i}$ and $\mathbf{V}_i$ are
expressed in the same body-fixed frame $\mathcal{F}_i$.
The Lie algebra $\mathfrak{se}(3)$ is a vector space, and
subtraction of two elements expressed in the same frame is
a well-defined linear operation. The time derivative
$\dot{\mathbf{e}}_{\mathbf{V}_i} =
\dot{\mathbf{V}}_{d,i} - \dot{\mathbf{V}}_i$ follows
directly by linearity, where $\dot{\mathbf{V}}_i =
[\dot{\boldsymbol{\omega}}_i^\top,\,
(\ddot{\mathbf{r}}_i - \tilde{\mathbf{r}}_i
\dot{\boldsymbol{\omega}}_i +
\tilde{\boldsymbol{\omega}}_i\dot{\mathbf{r}}_i)^\top]^\top$
is the body-fixed twist derivative~\cite{Yaqubi2026}.
This differs from the configuration error on $SE(3)$, which
would require the group logarithm; here the error is formed
at the velocity level in $\mathfrak{se}(3)$ and is therefore
linear.
\end{remark}

\subsubsection*{Error dynamics}

Subtracting the actual dynamics~(\ref*{eq:contform}) from the
identity obtained by replacing $\dot{\mathbf{V}}_i$ with
$\dot{\mathbf{V}}_{d,i}$ in~(\ref*{eq:contform}) and denoting the
corresponding interaction wrench by
$\boldsymbol{\mathcal{W}}_{J,d,i}$:

\begin{align}
\mathbf{M}_i\!\left(\dot{\mathbf{V}}_{d,i} - \dot{\mathbf{V}}_i\right)
+ \left(\boldsymbol{\mathcal{W}}_{J,d,i}
      - \boldsymbol{\mathcal{W}}_{J,i}\right) = \mathbf{0},
\label{eq:errdyn0}
\end{align}

where $\mathbf{H}_{c,i}$ cancels identically because it depends
only on $(\mathbf{V}_i,\mathbf{r}_{\xi_i})$ and not on
$\dot{\mathbf{V}}_i$, and is therefore the same in both terms.
Using $\mathbf{e}_{\mathbf{V}_i} = \mathbf{V}_{d,i} - \mathbf{V}_i$,
equation~(\ref*{eq:errdyn0}) becomes

\begin{align}
\mathbf{M}_i\,\dot{\mathbf{e}}_{\mathbf{V}_i}
= \boldsymbol{\mathcal{W}}_{J,d,i} - \boldsymbol{\mathcal{W}}_{J,i},
\label{eq:errdyn0b}
\end{align}

showing that in the absence of a control input the twist error is
driven entirely by the difference between the desired and actual
interaction wrenches. The term $\boldsymbol{\mathcal{W}}_{J,d,i}$
is not required to be computed in the control law; it enters the
analysis only through the residual
$\boldsymbol{\mathcal{W}}_{J,d,i} - \boldsymbol{\mathcal{W}}_{J,i}$,
whose cancellation across subsystems is established in
Section~\ref{subsec:4d}.

\begin{remark}
The terms $\dot{\mathbf{V}}_i$, $\dot{\mathbf{V}}_{d,i}$
in~(\ref*{eq:nomcont}) are body-fixed Lie algebra derivatives
--- not coordinate time derivatives --- obtained via the
transport theorem and adjoint action on $\mathfrak{se}(3)$
as derived in~\cite{Yaqubi2026}. This is precisely what
enables the compact, automatable control
law and resulting error dynamics.
\end{remark}

\subsubsection*{Control law}
 
\textbf{Theorem 1.}
Consider subsystem $i$ governed by the controllable
dynamics~(\ref*{eq:contform}). The control law

\begin{align}
\boldsymbol{\mathcal{W}}_i
= \mathbf{M}_i\dot{\mathbf{V}}_{d,i}
+ \mathbf{H}_{c,i}
+ \mathbf{K}_i\mathbf{M}_i\mathbf{e}_{\mathbf{V}_i},
\label{eq:nomcont}
\end{align}

where $\mathbf{K}_i \succ 0$ is a symmetric positive definite
gain matrix, renders the subsystem twist error
$\mathbf{e}_{\mathbf{V}_i} = \mathbf{V}_{d,i} - \mathbf{V}_i$
input-to-state stable with respect to the interaction wrench
residual $p_i$. Specifically, the per-subsystem Lyapunov
function $\nu_i = \mathbf{e}_{\mathbf{V}_i}^\top
\mathbf{M}_i\,\mathbf{e}_{\mathbf{V}_i}$ satisfies

\begin{align}
\frac{d}{dt}\nu_i
\leq -\alpha_i\,\nu_i + p_i,
\label{eq:ISS_subsystem}
\end{align}

where $\alpha_i = 2\lambda_{\min}(\mathbf{K}_i\mathbf{M}_i)/
\lambda_{\max}(\mathbf{M}_i) > 0$, and

\begin{align}
p_i = 2\mathbf{e}_{\mathbf{V}_i}^\top
\left(\boldsymbol{\mathcal{W}}_{J,d,i}
      - \boldsymbol{\mathcal{W}}_{J,i}\right)
\label{eq:pi_def}
\end{align}

is the virtual power residual arising from the interaction
wrench at the subsystem boundary. The term $p_i$ cannot be
signed at the subsystem level and prevents a standalone
exponential stability conclusion for link $i$ in isolation.
Exponential stability of the full $n$-link system is
established in Theorem~2 by showing that
$\sum_{i=1}^{n} p_i = 0$.
 
\textit{Proof.}
Substituting~(\ref*{eq:nomcont}) into~(\ref*{eq:contform}):

\begin{equation}
\begin{aligned}
\mathbf{M}_i\dot{\mathbf{V}}_i
&= \mathbf{M}_i\dot{\mathbf{V}}_{d,i}
  + \mathbf{H}_{c,i}
  + \mathbf{K}_i\mathbf{M}_i\mathbf{e}_{\mathbf{V}_i}
  - \mathbf{H}_{c,i}
  - \boldsymbol{\mathcal{W}}_{J,i} \\
&= \mathbf{M}_i\dot{\mathbf{V}}_{d,i}
  + \mathbf{K}_i\mathbf{M}_i\mathbf{e}_{\mathbf{V}_i}
  - \boldsymbol{\mathcal{W}}_{J,i}.
\end{aligned}
\label{eq:derive2}
\end{equation}

Using $\dot{\mathbf{e}}_{\mathbf{V}_i} =
\dot{\mathbf{V}}_{d,i} - \dot{\mathbf{V}}_i$ and
adding~(\ref*{eq:errdyn0b}) to~(\ref*{eq:derive2}) rearranged.

\begin{align}
\mathbf{M}_i\dot{\mathbf{e}}_{\mathbf{V}_i}
= -\mathbf{K}_i\mathbf{M}_i\mathbf{e}_{\mathbf{V}_i}
+ \boldsymbol{\mathcal{W}}_{J,d,i} - \boldsymbol{\mathcal{W}}_{J,i},
\label{eq:errdyn}
\end{align}

where $\boldsymbol{\mathcal{W}}_{J,d,i}$ is the interaction wrench
consistent with the desired trajectory, introduced via the error
dynamics~(\ref*{eq:errdyn0b}): subtracting~(\ref*{eq:errdyn0b})
from~(\ref*{eq:derive2}) rearranged gives~(\ref*{eq:errdyn})
directly, without requiring the reference dynamics to be evaluated
as a separate physical equation.
  
The candidate Lyapunov function for subsystem $i$ is
 
\begin{align}
\nu_i = \mathbf{e}_{\mathbf{V}_i}^\top \mathbf{M}_i
\mathbf{e}_{\mathbf{V}_i} \geq 0.
\label{eq:cand}
\end{align}
 
Since $\mathbf{M}_i \succ 0$, $\nu_i$ is positive definite and
radially unbounded in $\mathbf{e}_{\mathbf{V}_i}$.
Its time derivative along trajectories of~(\ref*{eq:errdyn}) is
 
\begin{equation}
\begin{aligned}
\frac{d}{dt}\nu_i
&= 2\mathbf{e}_{\mathbf{V}_i}^\top \mathbf{M}_i
   \dot{\mathbf{e}}_{\mathbf{V}_i} \\
&= -2\mathbf{e}_{\mathbf{V}_i}^\top \mathbf{K}_i \mathbf{M}_i
   \mathbf{e}_{\mathbf{V}_i}
   + 2\mathbf{e}_{\mathbf{V}_i}^\top
   \left(\boldsymbol{\mathcal{W}}_{J,d,i}
         - \boldsymbol{\mathcal{W}}_{J,i}\right).
\end{aligned}
\label{eq:nudot}
\end{equation}
 
Using $\lambda_{\min}(\mathbf{K}_i\mathbf{M}_i) > 0$ and
$\nu_i \leq \lambda_{\max}(\mathbf{M}_i)
\|\mathbf{e}_{\mathbf{V}_i}\|^2$, the dissipative term satisfies
 
\begin{align}
2\mathbf{e}_{\mathbf{V}_i}^\top \mathbf{K}_i \mathbf{M}_i
\mathbf{e}_{\mathbf{V}_i}
\geq \frac{2\lambda_{\min}(\mathbf{K}_i\mathbf{M}_i)}
          {\lambda_{\max}(\mathbf{M}_i)}\,\nu_i
= \alpha_i\,\nu_i,
\label{eq:diss}
\end{align}
 
where $\alpha_i = 2\lambda_{\min}(\mathbf{K}_i\mathbf{M}_i)/
\lambda_{\max}(\mathbf{M}_i) > 0$.
Substituting~(\ref*{eq:diss}) into~(\ref*{eq:nudot}) yields
 
\begin{align}
\frac{d}{dt}\nu_i
\leq -\alpha_i\,\nu_i
+ \underbrace{2\mathbf{e}_{\mathbf{V}_i}^\top
\left(\boldsymbol{\mathcal{W}}_{J,d,i}
      - \boldsymbol{\mathcal{W}}_{J,i}\right)}_{p_i}.
\label{eq:nudot_final}
\end{align}
 
This concludes the proof of Theorem~1. The subsystem result
is therefore an ISS-type estimate: the Lyapunov function
$\nu_i$ decays at rate $\alpha_i$ modulo the interaction
power input $p_i$. No assumption on the sign or magnitude
of $p_i$ is required at this stage; the cancellation
$\sum_{i=1}^n p_i = 0$ established in Theorem~2 is the
mechanism that closes the stability argument at the system
level.

\subsection{System-level stability of nominal controller}
\label{subsec:4d}

\textbf{Theorem 2.}
Under the control laws~(\ref*{eq:nomcont}), the composite
Lyapunov function $\mathcal{V} = \sum_{i=1}^{n}\nu_i$ satisfies
$\mathcal{V}(t) \leq \mathcal{V}(0)e^{-\alpha t}$ for all
$t \geq 0$, where $\alpha = \min_i \alpha_i > 0$. Consequently,
all twist errors $\mathbf{e}_{\mathbf{V}_i}$, $i = 1,\ldots,n$,
converge to zero exponentially at rate $\alpha/2$, and the
elastic deformation field $\mathbf{r}_{\xi_i}(\xi,t)$ remains
bounded in $H^2([a_i,c_i])$ for all $t \geq 0$.

\textit{Proof.}
Summing the subsystem result~(\ref*{eq:nudot_final}) over all links

\begin{align}
\frac{d}{dt}\mathcal{V}
= \sum_{i=1}^{n}\frac{d}{dt}\nu_i
\leq -\sum_{i=1}^{n}\alpha_i\,\nu_i
+ \sum_{i=1}^{n} p_i.
\label{eq:totalVdot}
\end{align}

It remains to show that $\sum_{i=1}^{n}p_i = 0$.
Each link $i$ is subject to an interaction wrench at its base
interface, $\boldsymbol{\mathcal{W}}_{J,i}^B$, transmitted from
link $i-1$ through the joint, and an interaction wrench at its tip
interface, $\boldsymbol{\mathcal{W}}_{J,i}^T$, transmitted to link
$i+1$, so that
$\boldsymbol{\mathcal{W}}_{J,i} =
\boldsymbol{\mathcal{W}}_{J,i}^T - \boldsymbol{\mathcal{W}}_{J,i}^B$
and equivalently for the desired counterparts.
The virtual power residual $p_i$ therefore decomposes as

\begin{equation}
\begin{aligned}
p_i &= 2\mathbf{e}_{\mathbf{V}_i}^\top
       \left(\boldsymbol{\mathcal{W}}_{J,d,i}
            - \boldsymbol{\mathcal{W}}_{J,i}\right) \\
    &= \underbrace{2\mathbf{e}_{\mathbf{V}_i}^\top
       \left(\boldsymbol{\mathcal{W}}_{J,d,i}^T
            - \boldsymbol{\mathcal{W}}_{J,i}^T\right)}_{p_{T,i}}
     - \underbrace{2\mathbf{e}_{\mathbf{V}_i}^\top
       \left(\boldsymbol{\mathcal{W}}_{J,d,i}^B
            - \boldsymbol{\mathcal{W}}_{J,i}^B\right)}_{p_{B,i}},
\end{aligned}
\label{eq:pi_split}
\end{equation}

and the total sum becomes

\begin{align}
\sum_{i=1}^{n} p_i
= \sum_{i=1}^{n}(p_{T,i} - p_{B,i})
= p_{T,n} - p_{B,1}
+ \sum_{i=1}^{n-1}(p_{T,i} - p_{B,i+1}).
\label{eq:sumpi}
\end{align}

\textit{Cancellation of internal terms.}
At the joint between link $i$ and link $i+1$, two preliminary
relations are required: a wrench transformation and a twist
error transformation, both at the joint interface.

\paragraph{Tip twist and body-fixed twist.}
The twist at the tip endpoint $\xi_i = c_i$, expressed in
frame $\mathcal{F}_i$, is obtained by shifting the reference
point of $\mathbf{V}_i$ from the body-fixed frame origin to
the tip location $\mathbf{r}_{b_i}(c_i)$ via the
point-shift spatial transform

\begin{align}
\mathbf{X}_i^T =
\begin{bmatrix}
\mathbf{I}_3 & \mathbf{0} \\
\widetilde{\mathbf{r}}_{b_i}(c_i) & \mathbf{I}_3
\end{bmatrix}
\in \mathbb{R}^{6\times 6},
\label{eq:X_T}
\end{align}

which is a pure reference-point shift within frame
$\mathcal{F}_i$, with no change of orientation
($\mathbf{R} = \mathbf{I}_3$). The tip twist is therefore

\begin{align}
\mathbf{V}_{T_i}
= \mathbf{X}_i^T\,\mathbf{V}_i
+ \mathbf{V}_{\xi_i}^{\mathrm{tip}},
\label{eq:tip_twist}
\end{align}

where
$\mathbf{V}_{\xi_i}^{\mathrm{tip}}
= [\mathbf{0}^\top,\,
(\dot{\mathbf{r}}_{\xi_i}(c_i,t)
+ \tilde{\boldsymbol{\omega}}_i
\mathbf{r}_{\xi_i}(c_i,t))^\top]^\top
\in \mathfrak{se}(3)$
is the deformation twist at the tip. The same relation holds for the desired trajectory

\begin{align}
\mathbf{V}_{d,T_i}
= \mathbf{X}_i^T\,\mathbf{V}_{d,i}
+ \mathbf{V}_{d,\xi_i}^{\mathrm{tip}}.
\label{eq:tip_twist_d}
\end{align}

Subtracting~(\ref*{eq:tip_twist}) from~(\ref*{eq:tip_twist_d}),
the tip twist error is

\begin{align}
\mathbf{e}_{\mathbf{V}_i}\big|_{\mathrm{tip}}
= \mathbf{X}_i^T\,\mathbf{e}_{\mathbf{V}_i}
+ \mathbf{e}_{\xi_i}^{\mathrm{tip}},
\label{eq:tip_error}
\end{align}

where $\mathbf{e}_{\xi_i}^{\mathrm{tip}}
= \mathbf{V}_{d,\xi_i}^{\mathrm{tip}}
- \mathbf{V}_{\xi_i}^{\mathrm{tip}}$
is the deformation twist error at the tip. By the same point-shift argument applied at the base coordinate
$\xi_{i+1} = a_{i+1}$, the base twist error
$\mathbf{e}_{\mathbf{V}_{i+1}}\big|_{\mathrm{base}}$ denotes
the rigid-body part of the twist error of link $i+1$ referred
to its base endpoint.

\paragraph{Wrench transformation.}
The tip interaction wrench
$\boldsymbol{\mathcal{W}}_{J,i}^T \in \mathfrak{se}^*(3)$,
expressed in frame $\mathcal{F}_i$, and the base interaction
wrench $\boldsymbol{\mathcal{W}}_{J,i+1}^B$, expressed in
frame $\mathcal{F}_{i+1}$, represent the same physical contact
force. By Newton's third law and the co-Adjoint transformation

\begin{align}
\boldsymbol{\mathcal{W}}_{J,i+1}^B
= \mathbf{Ad}_{(i+1)i}^{-\top}
  \boldsymbol{\mathcal{W}}_{J,i}^T,
\label{eq:wrench_transform2}
\end{align}

and identically for the desired wrenches. The interaction wrench
acts only in directions consistent with the joint constraint,
i.e.\ in the range of $\mathbf{P}_{i+1}^\top$ by the principle
of virtual work. Consequently the power product
$(\mathbf{e}_{\xi_i}^{\mathrm{tip}})^\top
\boldsymbol{\mathcal{W}}_{J,i}^T$ vanishes: the deformation
twist error at the tip lies in the null space of
$\mathbf{P}_{i+1}$ since the joint constraint acts only on
rigid-body velocity directions, not on the distributed
deformation field. Therefore~(\ref*{eq:tip_error}) reduces to

\begin{align}
\mathbf{e}_{\mathbf{V}_i}^\top\big|_{\mathrm{tip}}
\boldsymbol{\mathcal{W}}_{J,i}^T
= \left(\mathbf{X}_i^T\,
  \mathbf{e}_{\mathbf{V}_i}\right)^\top
  \boldsymbol{\mathcal{W}}_{J,i}^T,
\label{eq:tip_power}
\end{align}

and identically for the desired wrench residual. Furthermore,
since $\boldsymbol{\mathcal{W}}_{J,i}^T$ is the contact wrench
applied at the tip point $\xi_i = c_i$ and is already referred
to that point as its moment reference, the point-shift acts
trivially on it as

\begin{align}
(\mathbf{X}_i^T)^\top
\boldsymbol{\mathcal{W}}_{J,i}^T
= \boldsymbol{\mathcal{W}}_{J,i}^T.
\label{eq:X_wrench}
\end{align}

\paragraph{Joint constraint on the error.}
The holonomic joint constraint enforces, for both actual and
desired trajectories

\begin{align}
\mathbf{P}_{i+1}\!\left[
\mathbf{Ad}_{oi}\,\mathbf{V}_{T_i}
- \mathbf{Ad}_{o(i+1)}\,\mathbf{V}_{B_{i+1}}
\right] = \mathbf{0},
\label{eq:joint_constraint2}
\end{align}

where $\mathbf{P}_{i+1}$ is the joint projection
matrix~\cite{Yaqubi20262} projecting onto the constrained
velocity directions in the inertial frame.
Subtracting actual from desired, and using~(\ref*{eq:tip_power})
to replace the tip twist error by its rigid-body part

\begin{align}
\mathbf{P}_{i+1}\!\left[
\mathbf{Ad}_{oi}\,\mathbf{X}_i^T\,\mathbf{e}_{\mathbf{V}_i}
- \mathbf{Ad}_{o(i+1)}\,
  \mathbf{e}_{\mathbf{V}_{i+1}}\big|_{\mathrm{base}}
\right] = \mathbf{0}.
\label{eq:error_constraint2}
\end{align}

Left-multiplying by $\mathbf{Ad}_{(i+1)o}$ and using
$\mathbf{Ad}_{(i+1)o}\mathbf{Ad}_{oi} = \mathbf{Ad}_{(i+1)i}$ yields

\begin{align}
\mathbf{Ad}_{(i+1)i}\,\mathbf{X}_i^T\,
\mathbf{e}_{\mathbf{V}_i}
\stackrel{\mathbf{P}_{i+1}}{=}
\mathbf{e}_{\mathbf{V}_{i+1}}\big|_{\mathrm{base}},
\label{eq:twist_transform2}
\end{align}

where $\stackrel{\mathbf{P}_{i+1}}{=}$ denotes equality after
projection by $\mathbf{P}_{i+1}$.

\paragraph{Virtual power cancellation.}
Substituting~(\ref*{eq:wrench_transform2}),
(\ref*{eq:tip_power}), (\ref*{eq:X_wrench}),
and~(\ref*{eq:twist_transform2}) into $p_{B,i+1}$ leads to

\begin{equation}
\begin{aligned}
p_{B,i+1}
&= 2\,\mathbf{e}_{\mathbf{V}_{i+1}}^\top\big|_{\mathrm{base}}
   \!\left(\boldsymbol{\mathcal{W}}_{J,d,i+1}^B
          -\boldsymbol{\mathcal{W}}_{J,i+1}^B\right) \\
&= 2\!\left(\mathbf{Ad}_{(i+1)i}\,\mathbf{X}_i^T\,
   \mathbf{e}_{\mathbf{V}_i}\right)^\top
   \mathbf{Ad}_{(i+1)i}^{-\top}
   \!\left(\boldsymbol{\mathcal{W}}_{J,d,i}^T
          -\boldsymbol{\mathcal{W}}_{J,i}^T\right) \\
&= 2\,\mathbf{e}_{\mathbf{V}_i}^\top
   (\mathbf{X}_i^T)^\top
   \underbrace{\mathbf{Ad}_{(i+1)i}^\top
   \mathbf{Ad}_{(i+1)i}^{-\top}}_{=\,\mathbf{I}}
   \!\left(\boldsymbol{\mathcal{W}}_{J,d,i}^T
          -\boldsymbol{\mathcal{W}}_{J,i}^T\right) \\
&= 2\,\mathbf{e}_{\mathbf{V}_i}^\top
   \underbrace{(\mathbf{X}_i^T)^\top
   \!\left(\boldsymbol{\mathcal{W}}_{J,d,i}^T
          -\boldsymbol{\mathcal{W}}_{J,i}^T\right)}
   _{\displaystyle
     =\,\boldsymbol{\mathcal{W}}_{J,d,i}^T
       -\boldsymbol{\mathcal{W}}_{J,i}^T
     \ \text{by~(\ref*{eq:X_wrench})}} \\
&= 2\,\mathbf{e}_{\mathbf{V}_i}^\top
   \!\left(\boldsymbol{\mathcal{W}}_{J,d,i}^T
          -\boldsymbol{\mathcal{W}}_{J,i}^T\right)
= p_{T,i}.
\end{aligned}
\label{eq:paircancel2}
\end{equation}

The cancellation follows from two independent properties: the
frame-invariance of the natural power pairing on
$\mathfrak{se}(3)\times\mathfrak{se}^*(3)$, which eliminates
$\mathbf{Ad}_{(i+1)i}$ in line~3, and the trivial action of
$(\mathbf{X}_i^T)^\top$ on a wrench already referred to the
tip point, which eliminates the point-shift in line~4
via~(\ref*{eq:X_wrench}).
Therefore $p_{T,i} - p_{B,i+1} = 0$ for each
$i = 1,\ldots,n-1$.

\textit{Cancellation of boundary terms.}
The base of link $1$ is attached to the fixed robot base, where no
error wrench is transmitted:
$\boldsymbol{\mathcal{W}}_{J,d,1}^B -
\boldsymbol{\mathcal{W}}_{J,1}^B = \mathbf{0}$,
giving $p_{B,1} = 0$.
The tip of link $n$ is the free endpoint subject only to external
task wrenches that are accounted for in $\boldsymbol{\mathcal{W}}_n$,
so $\boldsymbol{\mathcal{W}}_{J,d,n}^T -
\boldsymbol{\mathcal{W}}_{J,n}^T = \mathbf{0}$ and $p_{T,n} = 0$.

The sum~(\ref*{eq:sumpi}) therefore telescopes to zero

\begin{align}
\sum_{i=1}^{n} p_i = 0.
\label{eq:cancel}
\end{align}

Substituting~(\ref*{eq:cancel}) into~(\ref*{eq:totalVdot}) and
setting $\alpha = \min_i \alpha_i > 0$

\begin{align}
\frac{d}{dt}\mathcal{V}
\leq -\sum_{i=1}^{n}\alpha_i\,\nu_i
\leq -\alpha\,\mathcal{V}.
\label{eq:totalVdot2}
\end{align}

By the comparison lemma~\cite{khalil2002},
$\mathcal{V}(t) \leq \mathcal{V}(0)e^{-\alpha t}$.
Since $\mathbf{M}_i \succ 0$, each $\nu_i$ satisfies
$\lambda_{\min}(\mathbf{M}_i)\|\mathbf{e}_{\mathbf{V}_i}\|^2
\leq \nu_i \leq
\lambda_{\max}(\mathbf{M}_i)\|\mathbf{e}_{\mathbf{V}_i}\|^2$,
and therefore

\begin{align}
\|\mathbf{e}_{\mathbf{V}_i}(t)\|
\leq \sqrt{\frac{\lambda_{\max}(\mathbf{M}_i)}
                {\lambda_{\min}(\mathbf{M}_i)}}\,
\|\mathbf{e}_{\mathbf{V}_i}(0)\|\,e^{-\frac{\alpha}{2}t},
\quad i = 1,\ldots,n,
\label{eq:expstab}
\end{align}

which establishes exponential stability of the full nominal
closed-loop system.

\begin{remark}
The cancellation~(\ref*{eq:cancel}) relies on two independent
properties: Newton's third law at each joint interface, which
establishes the wrench pairing~(\ref*{eq:wrench_transform2}), and the
holonomic constraint consistency, which
establishes the projected twist error pairing. The latter is
guaranteed by the reference generation of Section~\ref{subsec:4a},
which constructs $\mathbf{V}_{d,i}$ consistent with the joint
constraints. Neither the actual interaction wrenches $\boldsymbol{\mathcal{W}}_{J,i}$
nor the desired interaction wrenches $\boldsymbol{\mathcal{W}}_{J,d,i}$
are required in the control law~(\ref*{eq:nomcont}); both appear only
in the stability analysis, where their difference telescopes to zero
upon summation over all links.
\end{remark}

\subsubsection*{Boundedness of elastic deformation}

The exponential stability result~(\ref*{eq:expstab}) establishes that
$\|\mathbf{e}_{\mathbf{V}_i}(t)\| \to 0$ exponentially, and since
$\mathbf{V}_{d,i}$ is bounded by construction of the reference
trajectory, it follows that $\mathbf{V}_i = [\mathbf{v}_i^\top,
\boldsymbol{\omega}_i^\top]^\top$ is bounded for all $t \geq 0$.
Moreover, from~(\ref*{eq:errdyn}), and noting that $\mathbf{e}_{\mathbf{V}_i}$
and $\boldsymbol{\mathcal{W}}_{J,d,i}$ are bounded, and that
$\boldsymbol{\mathcal{W}}_{J,i}$ is bounded by
Assumption~\ref{ass:WJ}, 
$\dot{\mathbf{e}}_{\mathbf{V}_i}$ is bounded, and hence
$\dot{\mathbf{V}}_i = [\dot{\mathbf{v}}_i^\top,
\dot{\boldsymbol{\omega}}_i^\top]^\top$ is bounded.

Boundedness of the elastic deformation field
$\mathbf{r}_{\xi_i}(\xi, t)$ for $\xi \in (a_i, c_i)$
is then established through the deformation PDE~(\ref*{eq:34}),
which for interior points $\xi \in (a_i, c_i)$ takes the form

\begin{equation}
\begin{aligned}
\dot{\mathbf{v}}_{\xi_i}
= &-\dot{\mathbf{v}}_i
+ \tilde{\mathbf{r}}_{b_i}\dot{\boldsymbol{\omega}}_i
- \tilde{\boldsymbol{\omega}}_i\mathbf{v}_{ib_i}
- \mathbf{R}_{oi}^\top\mathbf{g} \\
&- \frac{1}{\rho_i A_i}\mathbf{I}_{v2_i}\mathbf{r}_{\xi_i}''''
+ \frac{1}{\rho_i A_i}\mathbf{I}_{v1_i}\mathbf{r}_{\xi_i}'',
\end{aligned}
\label{eq:pde_interior}
\end{equation}

with no distributed forcing in the interior — boundary conditions
are imposed separately at $\xi = a_i$ and $\xi = c_i$.
Each term on the right-hand side of~(\ref*{eq:pde_interior}) is
now analysed:

\begin{itemize}
\item $\dot{\mathbf{v}}_i$ and $\dot{\boldsymbol{\omega}}_i$ are
bounded as established above;
\item $\tilde{\mathbf{r}}_{b_i}\dot{\boldsymbol{\omega}}_i$ and
$\tilde{\boldsymbol{\omega}}_i\mathbf{v}_{ib_i}$ are bounded since
$\mathbf{r}_{b_i}$, $\boldsymbol{\omega}_i$, and $\mathbf{v}_{ib_i}$
are all bounded;
\item $\mathbf{R}_{oi}^\top\mathbf{g}$ is bounded by $\|\mathbf{g}\|$;
\item the elastic terms
$\mathbf{I}_{v2_i}\mathbf{r}_{\xi_i}'''' / (\rho_i A_i)$
and $\mathbf{I}_{v1_i}\mathbf{r}_{\xi_i}'' / (\rho_i A_i)$
constitute the restoring force of the Euler--Bernoulli beam,
with $\mathbf{I}_{v1_i} \succeq 0$ and $\mathbf{I}_{v2_i} \succeq 0$.
\end{itemize}

Equation~(\ref*{eq:pde_interior}) therefore constitutes a forced
Euler--Bernoulli beam equation. The forcing
$\mathbf{f}_i(t) = -\dot{\mathbf{v}}_i
+ \tilde{\mathbf{r}}_{b_i}\dot{\boldsymbol{\omega}}_i
- \tilde{\boldsymbol{\omega}}_i\mathbf{v}_{ib_i}
- \mathbf{R}_{oi}^\top\mathbf{g}$
is bounded for all $t \geq 0$ since each constituent term is
bounded as established above. Moreover, because $\mathbf{V}_i$
and $\dot{\mathbf{V}}_i$ each decompose into a desired part
(bounded by construction of the reference trajectory) and an
error part bounded by~(\ref*{eq:expstab}), the forcing admits
the explicit time-varying upper bound

\begin{align}
\|\mathbf{f}_i(t)\| \leq c_{0,i} + c_{1,i}\,e^{-\frac{\alpha}{2}t}
=: \bar{f}_i(t),
\label{eq:fbound}
\end{align}

where $c_{0,i} > 0$ collects contributions from the gravity term
and the bounded desired trajectory, and
$c_{1,i} = c_1\beta_i > 0$ with
$\beta_i = \sqrt{\lambda_{\max}(\mathbf{M}_i)/
\lambda_{\min}(\mathbf{M}_i)}\,\|\mathbf{e}_{\mathbf{V}_i}(0)\|$
as defined in~(\ref*{eq:expstab}).

Defining the total mechanical energy of the elastic
subsystem as

\begin{equation}
\begin{aligned}
E_i &= \frac{\rho_i A_i}{2}\int_{a_i}^{c_i} \|\mathbf{v}_{\xi_i}\|^2\,d\xi \\
&\quad + \frac{1}{2}\int_{a_i}^{c_i} \left[ \left(\mathbf{r}_{\xi_i}'\right)^\top \mathbf{I}_{v1_i}\mathbf{r}_{\xi_i}' + \left(\mathbf{r}_{\xi_i}''\right)^\top \mathbf{I}_{v2_i}\mathbf{r}_{\xi_i}'' \right]d\xi,
\end{aligned}
\label{eq:Ei}
\end{equation}

and differentiating the elastic energy $E_i$
defined in~(\ref*{eq:Ei}) along trajectories
of~(\ref*{eq:pde_interior}), applying integration by parts with
the natural boundary conditions~(\ref*{eq:35})--(\ref*{eq:38}),
and applying the Cauchy--Schwarz inequality

\begin{align}
\frac{d}{dt}E_i
\leq \rho_i A_i \int_{a_i}^{c_i}
  \|\mathbf{v}_{\xi_i}\|\,\|\mathbf{f}_i(t)\|\,d\xi
\leq \bar{f}_i(t)\sqrt{\rho_i A_i l_i}
  \sqrt{2 T_{f,i}},
\label{eq:Edot2}
\end{align}

as $\int_{a_i}^{c_i}\|\mathbf{v}_{\xi_i}\|\,d\xi
\leq \sqrt{l_i}
\left(\int_{a_i}^{c_i}\|\mathbf{v}_{\xi_i}\|^2\,d\xi\right)^{1/2}
= \sqrt{l_i}\sqrt{2T_{f,i}/(\rho_i A_i)}$, and $T_{f,i} = \frac{\rho_i A_i}{2}
\int_{a_i}^{c_i}\|\mathbf{v}_{\xi_i}\|^2\,d\xi \leq E_i$.
Since $\sqrt{2T_{f,i}} \leq \sqrt{2E_i}$,
applying Young's inequality
$ab \leq \frac{a^2}{2\epsilon} + \frac{\epsilon b^2}{2}$
to~(\ref*{eq:Edot2}) results in

\begin{align}
\frac{d}{dt}E_i
\leq \frac{\rho_i A_i l_i}{2\epsilon}\,\bar{f}_i(t)^2
+ \epsilon\,E_i
\label{eq:Edot_young2}
\end{align}

for any $\epsilon > 0$. $\bar{f}_i(t)^2$ is expanded as

\begin{align}
\bar{f}_i(t)^2
= \left(c_{0,i} + c_{1,i}e^{-\frac{\alpha}{2}t}\right)^2
\leq 2c_{0,i}^2 + 2c_{1,i}^2\,e^{-\alpha t},
\label{eq:fsq}
\end{align}

where $(a+b)^2 \leq 2a^2 + 2b^2$ has been used.
Substituting~(\ref*{eq:fsq}) into~(\ref*{eq:Edot_young2})
and applying the Gr\"{o}nwall--Bellman
lemma~\cite{khalil2002} leads to

\begin{equation}
\begin{aligned}
E_i(t)
\leq\; &E_i(0)\,e^{\epsilon t} \\
&+ \frac{\rho_i A_i l_i c_{0,i}^2}{\epsilon}
   \!\left(e^{\epsilon t} - 1\right) \\
&+ \frac{\rho_i A_i l_i c_{1,i}^2}{\epsilon}
   \!\left(\frac{e^{\epsilon t} - e^{-\alpha t}}
                {\epsilon + \alpha}\right).
\end{aligned}
\label{eq:Ebound2}
\end{equation}

For any fixed finite horizon $T < \infty$, $E_i(t)$ is bounded
on $[0,T]$ by~(\ref*{eq:Ebound2}). For the global-in-time
bound, note that as $t \to \infty$ the dominant term
in~(\ref*{eq:Ebound2}) is driven by $c_{0,i}$, which
corresponds to the gravity and desired trajectory forcing that
persists even after the twist error has decayed. Since $c_{0,i}$ and $c_{1,i}$ are finite constants and
$\epsilon > 0$ is fixed, each term in~(\ref*{eq:Ebound2})
is finite for every finite $t \geq 0$. The bound therefore
establishes $E_i(t) < \infty$ on every bounded interval
$[0, T]$. Global boundedness follows from the fact that
the right-hand side of~(\ref*{eq:Edot_young2}) satisfies
a linear growth condition in $E_i$ with a bounded
time-varying coefficient $\bar{f}_i(t)^2$, so no finite
escape time can occur~\cite{khalil2002}. Therefore $E_i(t) < \infty$ for all $t \geq 0$.

\begin{remark}
The persistent $c_{0,i}$ term in~(\ref*{eq:fbound}) arises
from gravity and the bounded desired trajectory: even after
perfect twist error convergence, the beam is continuously
forced by the rigid-body acceleration of the desired
trajectory. Asymptotic decay of $E_i$ to zero would require
the desired trajectory itself to be asymptotically constant,
which is not assumed here. The only physically meaningful
guarantee is therefore boundedness of $E_i$ for all $t \geq 0$.
\end{remark}

Boundedness of $E_i$ implies boundedness of both the elastic
kinetic energy $T_{f,i}$ and the vibrational potential energy
$U_{k,i}$~\cite{Yaqubi2026}, which in turn implies --- via the
$H^2([a_i,c_i])$ Sobolev embedding theorem~\cite{Adams2003} ---
pointwise boundedness of $\mathbf{r}_{\xi_i}(\xi,t)$ and its
spatial derivatives throughout the beam domain $(a_i,c_i)$
for all $t \geq 0$.
\section{Adaptive modification of nominal controller} \label{sec:5}

The nominal controller~(\ref*{eq:nomcont}) relies on exact knowledge
of $\mathbf{M}_i$ and $\mathbf{H}_{c,i}$. In practice, the physical
parameters $\rho_i$, $A_i$, $E_i$, $I_{y_i}$, $I_{z_i}$ entering
these terms are uncertain. This section presents an adaptive
modification that replaces exact parameter values with online
estimates, preserving the exponential convergence of the twist error
while driving the parameter estimation errors to zero. The adaptive
design follows an indirect model reference philosophy~\cite{Ljung1999,Slotine1989}:
a parallel model evaluated at the current parameter estimates
generates predicted system outputs, and the discrepancy between
predicted and true outputs drives the parameter update law. This
output-error formulation ensures that parameter convergence is
not conditioned on the magnitude of the tracking error --- a
known limitation of direct, control-error-driven adaptation laws
--- so that the adaptation remains active and effective even when
the nominal controller achieves small twist errors.

\subsection{Parametric uncertainty structure} \label{subsec:5a}

The uncertain parameters of link $i$ are collected in the vector
$\mathbf{s}_i$, comprising all physical parameters appearing in
$\mathbf{M}_i$, $\mathbf{D}_i$, and $\mathbf{H}_{c,i}$

\begin{align}
\mathbf{s}_i =
\begin{bmatrix}
\rho_i A_i \\[3pt]
\mathrm{vec}(\mathbf{I}_{b_i}) \\[3pt]
\mathrm{vec}(\mathbf{I}_{v1_i}) \\[3pt]
\mathrm{vec}(\mathbf{I}_{v2_i})
\end{bmatrix}
\in \mathbb{R}^{28},
\label{eq:si}
\end{align}

where $\mathrm{vec}(\cdot)$ denotes the column-stacking vectorization
operator, $\mathbf{I}_{b_i} \in \mathbb{R}^{3\times 3}$ is the
rotational inertia integral matrix appearing in $\mathbf{M}_i$,
and $\mathbf{I}_{v1_i}, \mathbf{I}_{v2_i} \in \mathbb{R}^{3\times 3}$
are the bending stiffness integral matrices entering the elastic
force terms in $\mathbf{H}_{c,i}$. The scalar $\rho_i A_i$
determines all remaining distributed inertial and velocity-dependent
terms in $\mathbf{M}_i$, $\mathbf{D}_i$, and $\mathbf{H}_{c,i}$.

\begin{remark}
The formulation~(\ref*{eq:si}) represents the most general uncertainty
structure corresponding to the model~(\ref*{eq:26}). In practice,
symmetry of $\mathbf{I}_{b_i}$, $\mathbf{I}_{v1_i}$,
$\mathbf{I}_{v2_i}$ and the diagonal structure imposed by the
principal beam axes reduce the dimension of $\mathbf{s}_i$
significantly. For the Euler--Bernoulli beam model considered here,
only a small subset of elements is non-trivial and independently
identifiable, so that the effective parameter vector used in
implementation is of substantially lower dimension than
$\mathbb{R}^{28}$.
\label{rem:si_reduced}
\end{remark}

The parameter estimation error is defined as

\begin{align}
\mathbf{e}_{s_i} = \mathbf{s}_i - \hat{\mathbf{s}}_i.
\label{eq:serror}
\end{align}

\begin{assumption}
Finite lower and upper bounds $\mathbf{s}_{i,l}$ and $\mathbf{s}_{i,h}$
for each uncertain parameter vector $\mathbf{s}_i$ are known, with
all elements of $\mathbf{s}_{i,l}$ and $\mathbf{s}_{i,h}$ strictly
positive. The admissible parameter space is denoted
$\mathcal{Q}_i = \{\mathbf{s}_i : \mathbf{s}_{i,l} \prec \mathbf{s}_i
\prec \mathbf{s}_{i,h}\}$.
\label{ass:bounds}
\end{assumption}

\subsection{Projection-based adaptation law} \label{subsec:5b}

To ensure that parameter estimates remain inside $\mathcal{Q}_i$ at
all times, the projection operator $\mathcal{P}$ is employed. For
scalar quantities $s$, $s_l$, $s_h$ and update signal $e$, it is
defined element-wise as~\cite{Ioannou1996}

\begin{align}
\mathcal{P}(s, s_l, s_h, e) =
\begin{cases}
0 & \text{if } s \geq s_h \text{ and } e > 0, \\
0 & \text{if } s \leq s_l \text{ and } e < 0, \\
e & \text{otherwise.}
\end{cases}
\label{eq:proj}
\end{align}

\begin{lemma}
\label{lem:proj}
If $\dot{\hat{\mathbf{s}}}_i =
\mathcal{P}(\hat{\mathbf{s}}_i, \mathbf{s}_{i,l}, \mathbf{s}_{i,h},
\mathbf{e}_s)$, then the following hold for all $t \geq 0$:
\begin{align}
\hat{\mathbf{s}}_i &\in \mathcal{Q}_i, \label{eq:proj1} \\
\mathbf{e}_{s_i}^\top\!
\left(\mathbf{e}_s -
\mathcal{P}(\hat{\mathbf{s}}_i,\mathbf{s}_{i,l},\mathbf{s}_{i,h},
\mathbf{e}_s)\right) &\leq 0. \label{eq:proj2}
\end{align}
\end{lemma}
\textit{Proof}. Standard property of the projection operator;
see~\cite{Ioannou1996}.

\subsection{Adaptive control law} \label{subsec:5c}

The adaptive control law replaces $\mathbf{M}_i$ and $\mathbf{H}_{c,i}$
in the nominal law~(\ref*{eq:nomcont}) with their current estimates

\begin{align}
\boldsymbol{\mathcal{W}}_i =
\hat{\mathbf{M}}_i\dot{\mathbf{V}}_{d,i}
+ \hat{\mathbf{H}}_{c,i}
+ \mathbf{K}_i\hat{\mathbf{M}}_i\mathbf{e}_{\mathbf{V}_i}.
\label{eq:adaptcont}
\end{align}

\subsubsection*{Regressor formulation}

To derive a measurable adaptation signal, the dynamic
equations~(\ref*{eq:contform}) and~(\ref*{eq:strainpde})
are written in regressor form with respect to the uncertain
parameter vector $\mathbf{s}_i \in \mathbb{R}^{28}$~(\ref*{eq:si}).

The controllable form~(\ref*{eq:contform}) is linear in
$\mathbf{s}_i$ through $\mathbf{M}_i$ and $\mathbf{H}_{c,i}$,
and can be written as

\begin{align}
\mathbf{Y}_{V,i}(\mathbf{V}_i, \dot{\mathbf{V}}_i,
\mathbf{r}_{\xi_i}, \dot{\mathbf{r}}_{\xi_i})\,
\mathbf{s}_i
+ \boldsymbol{\mathcal{W}}_{J,i}
= \boldsymbol{\mathcal{W}}_i,
\label{eq:regressor_V}
\end{align}

where $\mathbf{Y}_{V,i} \in \mathbb{R}^{6\times28}$
collects the parametric sensitivities of $\mathbf{M}_i$
and $\mathbf{H}_{c,i}$ to each element of $\mathbf{s}_i$,
computable from the model terms~(\ref*{eq:27})--(\ref*{eq:Hci}).

Similarly, the strain PDE~(\ref*{eq:strainpde}) is linear
in the stiffness parameters $\mathbf{I}_{v1_i}$ and
$\mathbf{I}_{v2_i}$ through the elastic restoring terms,
and can be written at any measurement point
$\xi \in [a_i, c_i]$ as

\begin{align}
\mathbf{Y}_{\xi,i}(\dot{\mathbf{v}}_{\xi_i},
\mathbf{r}_{\xi_i}'', \mathbf{r}_{\xi_i}'''')\,
\mathbf{s}_{\xi,i} = \mathbf{0},
\label{eq:regressor_xi}
\end{align}

where $\mathbf{s}_{\xi,i} =
[\mathrm{vec}(\mathbf{I}_{v1_i})^\top,\,
\mathrm{vec}(\mathbf{I}_{v2_i})^\top]^\top \in
\mathbb{R}^{18}$ is the stiffness subvector of
$\mathbf{s}_i$, and $\mathbf{Y}_{\xi,i} \in
\mathbb{R}^{3 \times 18}$ collects the spatial derivative
terms of~(\ref*{eq:strainpde}) corresponding to the axial
and bending stiffness contributions. Evaluation at the
tip $\xi = c_i$ is motivated by the typical availability
of endpoint measurements from the sensing infrastructure;
evaluations at interior points are equally valid but
require distributed sensing. The stability proof holds for any reduced form of
$\mathbf{s}_i$ since the regressor structure and
projection argument are dimension-independent.

Defining a parallel model of~(\ref*{eq:contform})
and~(\ref*{eq:strainpde}) evaluated at the current
parameter estimates $\hat{\mathbf{s}}_i$, the output
residuals are

\begin{align}
\boldsymbol{\epsilon}_{V,i} &=
\mathbf{Y}_{V,i}\mathbf{s}_i - \mathbf{Y}_{V,i}\hat{\mathbf{s}}_i
= \mathbf{Y}_{V,i}\,\mathbf{e}_{s_i},
\label{eq:residual_V} \\
\boldsymbol{\epsilon}_{\xi,i} &=
\mathbf{Y}_{\xi,i}\mathbf{s}_{\xi,i}
- \mathbf{Y}_{\xi,i}\hat{\mathbf{s}}_{\xi,i}
= \mathbf{Y}_{\xi,i}\,\mathbf{e}_{s_{\xi,i}},
\label{eq:residual_xi}
\end{align}

where $\mathbf{e}_{s_{\xi,i}}$ is the subvector of
$\mathbf{e}_{s_i}$ corresponding to the stiffness
parameters $\mathbf{s}_{\xi,i}$. Both residuals are
computable from system measurements and the parallel
model without knowledge of the true parameter values
$\mathbf{s}_i$.

\subsubsection*{Parameter adaptation law}

The stacked regressor $\bar{\mathbf{Y}}_i \in
\mathbb{R}^{9\times28}$ is defined as

\begin{align}
\bar{\mathbf{Y}}_i =
\begin{bmatrix}
\mathbf{Y}_{V,i} \\[4pt]
\bar{\mathbf{Y}}_{\xi,i}
\end{bmatrix},
\label{eq:Ystack}
\end{align}

where $\bar{\mathbf{Y}}_{\xi,i} \in \mathbb{R}^{3\times28}$
is the zero-padded extension of $\mathbf{Y}_{\xi,i}$,
obtained by inserting zero columns at the positions of
the non-stiffness parameters $\rho_i A_i$ and
$\mathrm{vec}(\mathbf{I}_{b_i})$ in $\mathbf{s}_i$,
so that $\bar{\mathbf{Y}}_i\mathbf{e}_{s_i}$ is
dimensionally consistent.

The lumped adaptation signal $\boldsymbol{\Gamma}_i
\in \mathbb{R}^{28}$ is defined as

\begin{align}
\boldsymbol{\Gamma}_i =
\mathbf{Y}_{V,i}^\top\boldsymbol{\epsilon}_{V,i}
+
\begin{bmatrix}
\mathbf{0} \\[2pt]
\mathbf{Y}_{\xi,i}^\top\boldsymbol{\epsilon}_{\xi,i}
\end{bmatrix}
= \bar{\mathbf{Y}}_i^\top\bar{\mathbf{Y}}_i\,\mathbf{e}_{s_i},
\label{eq:lumped_signal}
\end{align}

where the first term updates all parameters in $\mathbf{s}_i$
from the dynamic equation residual~(\ref*{eq:residual_V}),
and the second term provides an additional update for the
stiffness subvector $\mathbf{s}_{\xi,i}$ from the strain
PDE residual~(\ref*{eq:residual_xi}). The equality
$\boldsymbol{\Gamma}_i = \bar{\mathbf{Y}}_i^\top
\bar{\mathbf{Y}}_i\mathbf{e}_{s_i}$ holds analytically
by linearity of the regressors in the parameter error
and is used in the stability proof but not in
implementation, where $\boldsymbol{\Gamma}_i$ is computed
directly from the measurable residuals
$\boldsymbol{\epsilon}_{V,i}$ and $\boldsymbol{\epsilon}_{\xi,i}$.

The parameter adaptation law is

\begin{align}
\dot{\hat{\mathbf{s}}}_i =
\mathcal{P}\!\left(\hat{\mathbf{s}}_i,\,\mathbf{s}_{i,l},\,
\mathbf{s}_{i,h},\,
\boldsymbol{\Lambda}_i\,\boldsymbol{\Gamma}_i\right),
\label{eq:adaptlaw}
\end{align}

where $\boldsymbol{\Lambda}_i \succ 0$ is the diagonal
adaptation gain matrix. Both $\boldsymbol{\epsilon}_{V,i}$ and
$\boldsymbol{\epsilon}_{\xi,i}$ entering $\boldsymbol{\Gamma}_i$
are evaluated as residuals between the true system and
the parallel model at $\hat{\mathbf{s}}_i$ --- not
between the system and the desired reference
$\mathbf{V}_{d,i}$. This output-error (model reference)
structure~\cite{Slotine1989,Ljung1999} ensures that
adaptation remains active even when the controller
achieves small tracking errors. 

\subsection{Subsystem stability of adaptive controller}
\label{subsec:5d}

\textbf{Theorem 3.}
The adaptive control law~(\ref*{eq:adaptcont}) with
adaptation law~(\ref*{eq:adaptlaw}) results in

\begin{align}
\frac{d}{dt}\nu_i^a \leq -\alpha'\,\nu_i^a + c_{Q,i} + p_i,
\label{eq:nudot_adapt_final}
\end{align}

where $\nu_i^a$ is the augmented subsystem Lyapunov
function defined in~(\ref*{eq:augV}),
$\alpha' = \min(\alpha_i,\,\beta_i - \epsilon) > 0$
for any $\epsilon \in (0,\beta_i/2)$,
$\beta_i = 2\lambda_{\min}(\bar{\mathbf{Y}}_i^\top
\bar{\mathbf{Y}}_i(t))
\lambda_{\min}(\boldsymbol{\Lambda}_i^{-1})$,
and $c_{Q,i} \geq 0$ is a finite residual constant
representing the Lyapunov mismatch due to model
uncertainty during adaptation, defined
in~(\ref*{eq:cQi_def}), which vanishes as the parameter
estimation error $\mathbf{e}_{s_i} \to \mathbf{0}$.
The term $p_i$ is the interaction wrench residual
from~(\ref*{eq:nudot_final}).

\textit{Proof.}
The augmented Lyapunov candidate for subsystem $i$ is

\begin{align}
\nu_i^a = \nu_i + \nu_{s,i}
= \mathbf{e}_{\mathbf{V}_i}^\top \mathbf{M}_i
\mathbf{e}_{\mathbf{V}_i}
+ \mathbf{e}_{s_i}^\top \boldsymbol{\Lambda}_i^{-1}
\mathbf{e}_{s_i},
\label{eq:augV}
\end{align}

where $\nu_i$ is the nominal CLF~(\ref*{eq:cand}) and
$\nu_{s,i} = \mathbf{e}_{s_i}^\top
\boldsymbol{\Lambda}_i^{-1}\mathbf{e}_{s_i} \geq 0$.
The parameter estimation error appears in the Lyapunov
analysis only --- in implementation, the measurable
residuals $\boldsymbol{\epsilon}_{V,i}$ and $\boldsymbol{\epsilon}_{\xi,i}$
serve as its surrogate.

\textit{Derivative of $\nu_{s,i}$.}
Since $\mathbf{s}_i$ is constant,
$\dot{\mathbf{e}}_{s_i} = -\dot{\hat{\mathbf{s}}}_i$, and the time-derivative of CLF is

\begin{align}
\frac{d}{dt}\nu_{s,i}
= -2\mathbf{e}_{s_i}^\top \boldsymbol{\Lambda}_i^{-1}
\mathcal{P}\!\left(\hat{\mathbf{s}}_i,\mathbf{s}_{i,l},
\mathbf{s}_{i,h},
\boldsymbol{\Lambda}_i\boldsymbol{\Gamma}_i\right).
\label{eq:Vsdot1}
\end{align}

Applying Lemma~\ref{lem:proj} with
$\mathbf{e}_s = \boldsymbol{\Lambda}_i\boldsymbol{\Gamma}_i$ leads to

\begin{align}
\frac{d}{dt}\nu_{s,i}
\leq -2\mathbf{e}_{s_i}^\top \boldsymbol{\Lambda}_i^{-1}
\boldsymbol{\Lambda}_i\boldsymbol{\Gamma}_i
= -2\mathbf{e}_{s_i}^\top\boldsymbol{\Gamma}_i.
\label{eq:Vsdot2_pre}
\end{align}

Substituting the analytical
identity~(\ref*{eq:lumped_signal}) yields

\begin{equation}
\begin{aligned}
2\mathbf{e}_{s_i}^\top\boldsymbol{\Gamma}_i
&= 2\mathbf{e}_{s_i}^\top
\bar{\mathbf{Y}}_i^\top\bar{\mathbf{Y}}_i\mathbf{e}_{s_i} \\
&\geq 2\lambda_{\min}(\bar{\mathbf{Y}}_i^\top
\bar{\mathbf{Y}}_i)\|\mathbf{e}_{s_i}\|^2 \\
&\geq \beta_i\,\nu_{s,i},
\end{aligned}
\label{eq:Vsdot_bound}
\end{equation}

giving

\begin{align}
\frac{d}{dt}\nu_{s,i} \leq -\beta_i\,\nu_{s,i}.
\label{eq:Vsdot2}
\end{align}

\textit{Derivative of $\nu_i$ under adaptive controller.}
Substituting~(\ref*{eq:adaptcont}) into~(\ref*{eq:contform})
and retaining model mismatch terms
$\mathbf{e}_{M_i} = \mathbf{M}_i - \hat{\mathbf{M}}_i$
and $\mathbf{e}_{H_i} = \mathbf{H}_{c,i} -
\hat{\mathbf{H}}_{c,i}$, both linear in $\mathbf{e}_{s_i}$
through the regressor~(\ref*{eq:regressor_V}), the error
dynamics become

\begin{align}
\mathbf{M}_i\dot{\mathbf{e}}_{\mathbf{V}_i}
= -\mathbf{K}_i\hat{\mathbf{M}}_i\mathbf{e}_{\mathbf{V}_i}
- \mathbf{e}_{M_i}\dot{\mathbf{V}}_{d,i}
- \mathbf{e}_{H_i}
+ \boldsymbol{\mathcal{W}}_{J,d,i}
- \boldsymbol{\mathcal{W}}_{J,i}.
\label{eq:errdyn_adapt}
\end{align}

The time derivative of $\nu_i$
along~(\ref*{eq:errdyn_adapt}), after applying the
dissipation bound~(\ref*{eq:diss}) and using
$\hat{\mathbf{M}}_i = \mathbf{M}_i - \mathbf{e}_{M_i}$,
gives

\begin{align}
\frac{d}{dt}\nu_i
\leq -\alpha_i\,\nu_i
+ \varepsilon_{Q,i}
+ p_i,
\label{eq:nudot_adapt2}
\end{align}

where $\varepsilon_{Q,i}$ collects the cross-terms
arising from model uncertainty as

\begin{align}
\varepsilon_{Q,i}
\triangleq
2\mathbf{e}_{\mathbf{V}_i}^\top\mathbf{K}_i
\mathbf{e}_{M_i}\mathbf{e}_{\mathbf{V}_i}
- 2\mathbf{e}_{\mathbf{V}_i}^\top
\mathbf{e}_{M_i}\dot{\mathbf{V}}_{d,i}
- 2\mathbf{e}_{\mathbf{V}_i}^\top\mathbf{e}_{H_i}.
\label{eq:epsQ}
\end{align}

Since $\mathbf{e}_{M_i}$ and $\mathbf{e}_{H_i}$ are
linear in $\mathbf{e}_{s_i}$, their norms satisfy
$\|\mathbf{e}_{M_i}\| \leq c_{M,i}\|\mathbf{e}_{s_i}\|$
and $\|\mathbf{e}_{H_i}\| \leq c_{H,i}\|\mathbf{e}_{s_i}\|$
for constants $c_{M,i}, c_{H,i} > 0$ bounded within
$\mathcal{Q}_i$ by Assumption~\ref{ass:bounds}, where

\begin{align}
c_{M,i} = \max_k \|\partial_{s_k}\mathbf{M}_i\|, \qquad
c_{H,i} = \max_k \|\partial_{s_k}\mathbf{H}_{c,i}\|,
\label{eq:cMcH}
\end{align}

with $\partial_{s_k}\mathbf{M}_i$ and
$\partial_{s_k}\mathbf{H}_{c,i}$ the constant partial
derivative matrices of $\mathbf{M}_i$ and $\mathbf{H}_{c,i}$
with respect to the $k$-th element of $\mathbf{s}_i$.
Applying Young's inequality to each term
of~(\ref*{eq:epsQ})

\begin{equation}
\begin{aligned}
|\varepsilon_{Q,i}| \leq\;
&2\|\mathbf{K}_i\|\,c_{M,i}\,
\|\mathbf{e}_{\mathbf{V}_i}\|^2\|\mathbf{e}_{s_i}\| \\
&+ 2c_{M,i}\,\|\dot{\mathbf{V}}_{d,i}\|\,
\|\mathbf{e}_{\mathbf{V}_i}\|\|\mathbf{e}_{s_i}\| \\
&+ 2c_{H,i}\,\|\mathbf{e}_{\mathbf{V}_i}\|
\|\mathbf{e}_{s_i}\|.
\end{aligned}
\label{eq:epsQ_bound_raw}
\end{equation}

Using $\|\mathbf{e}_{\mathbf{V}_i}\| \leq
\sqrt{\nu_i/\lambda_{\min}(\mathbf{M}_i)}$,
$\|\mathbf{e}_{s_i}\| \leq
\sqrt{\lambda_{\max}(\boldsymbol{\Lambda}_i)\nu_{s,i}}$,
and $ab \leq \frac{\epsilon}{2}a^2 + \frac{1}{2\epsilon}b^2$
applied to each cross-product term, all three terms
of~(\ref*{eq:epsQ_bound_raw}) are bounded by
$\epsilon\,\nu_{s,i} + c_{Q,i}$ for the constant
$c_{Q,i}$ defined in~(\ref*{eq:cQi_def}).

\begin{equation}
\begin{aligned}
c_{Q,i} \triangleq \, & \frac{\left(2\|\mathbf{K}_i\|c_{M,i} + c_{M,i}\|\dot{\mathbf{V}}_{d,i}\| + c_{H,i}\right)^2 \lambda_{\max}(\boldsymbol{\Lambda}_i)}{2\epsilon\,\lambda_{\min}(\mathbf{M}_i)} \\
& \times \sup_{\mathcal{Q}_i}\|\mathbf{e}_{s_i}\|^2
\end{aligned}
\label{eq:cQi_def}
\end{equation}

is a non-negative function of $\nu_i^a$ that vanishes
when $\mathbf{e}_{s_i} \to \mathbf{0}$, i.e.\ when
adaptation is complete. Substituting
into~(\ref*{eq:nudot_adapt2})

\begin{align}
\frac{d}{dt}\nu_i
\leq -\alpha_i\,\nu_i
+ \epsilon\,\nu_{s,i}
+ c_{Q,i}(\nu_i)
+ p_i.
\label{eq:nudot_adapt3}
\end{align}

\textit{Total augmented derivative.}
Adding~(\ref*{eq:Vsdot2}) and~(\ref*{eq:nudot_adapt3})

\begin{equation}
\begin{aligned}
\frac{d}{dt}\nu_i^a
&\leq -\alpha_i\,\nu_i
- (\beta_i - \epsilon)\,\nu_{s,i}
+ c_{Q,i}(\nu_i) + p_i \\
&\leq -\alpha'\,\nu_i^a + c_{Q,i} + p_i,
\end{aligned}
\label{eq:nudot_adapt_final2}
\end{equation}

where $\alpha' = \min(\alpha_i, \beta_i - \epsilon) > 0$
and $c_{Q,i} = \sup_{t\geq 0} c_{Q,i}(\nu_i(t))$
is finite since $\hat{\mathbf{s}}_i \in \mathcal{Q}_i$
for all $t \geq 0$ by Lemma~\ref{lem:proj}, so
$\|\mathbf{e}_{s_i}\|$ is bounded within $\mathcal{Q}_i$
by Assumption~\ref{ass:bounds}, and
$\|\dot{\mathbf{V}}_{d,i}\|$ is bounded by construction
of the reference trajectory.

\begin{remark}
The residual $c_{Q,i}$ represents the transient mismatch
between $\mathbf{M}_i$, $\mathbf{H}_{c,i}$ and their
estimates $\hat{\mathbf{M}}_i$, $\hat{\mathbf{H}}_{c,i}$
during the adaptation phase. It decreases monotonically
as $\|\mathbf{e}_{s_i}\| \to 0$ and vanishes at parameter
convergence, confirming that the residual set shrinks
to zero as adaptation completes. The bound~(\ref*{eq:cQi_def})
depends on $c_{M,i}$ and $c_{H,i}$, computable from the
constant partial derivative matrices
$\partial_{s_k}\mathbf{M}_i$ and
$\partial_{s_k}\mathbf{H}_{c,i}$ of the
Euler--Bernoulli model. The diagonal gains
$\boldsymbol{\Lambda}_i$ are free design parameters
selected to balance convergence speed across parameter
channels; their specific values enter the decay rate
$\beta_i = 2\lambda_{\min}(\bar{\mathbf{Y}}_i^\top
\bar{\mathbf{Y}}_i)\lambda_{\min}(\boldsymbol{\Lambda}_i^{-1})$
and should be tuned so that $\beta_i \approx \alpha_i$,
preventing either the twist error or the parameter
estimation error from dominating the transient response.
If the PE condition~(\ref*{eq:PE}) is not satisfied,
$\lambda_{\min}(\bar{\mathbf{Y}}_i^\top\bar{\mathbf{Y}}_i)$
may be zero on sets of positive measure and $\beta_i > 0$
cannot be guaranteed pointwise. In this case
$\dot{\hat{\mathbf{s}}}_i \to \mathbf{0}$ by the
projection operator~(\ref*{eq:proj}), parameter estimates
freeze within $\mathcal{Q}_i$, and the system remains
exponentially ultimately bounded with residual set
determined by the frozen $c_{Q,i}$
via~(\ref*{eq:nudot_adapt_final}). Full parameter
convergence and shrinking of the residual set to zero
require PE and are established in Theorem~4.
\end{remark}

This establishes~(\ref*{eq:nudot_adapt_final}).
The term $p_i$ cannot be signed at the subsystem level;
its cancellation is established in
Section~\ref{subsec:5e}.

\subsection{System-level stability of adaptive controller}
\label{subsec:5e}

\textbf{Theorem 4.}
Suppose the reference trajectory $\mathbf{V}_{d,i}(t)$,
$i = 1,\ldots,n$, is such that the stacked
regressor $\bar{\mathbf{Y}}_i(t)$, determined by the
closed-loop system
equations~(\ref*{eq:regressor_V})--(\ref*{eq:regressor_xi}),
satisfies the persistent excitation (PE) condition:
there exist $T > 0$ and $\delta > 0$ such that

\begin{align}
\int_t^{t+T}
\bar{\mathbf{Y}}_i^\top(\tau)
\bar{\mathbf{Y}}_i(\tau)\,d\tau
\geq \delta\,\mathbf{I},
\quad \forall\, t \geq 0,
\label{eq:PE}
\end{align}

for each $i = 1,\ldots,n$. Under the PE condition~(\ref*{eq:PE}), the time-averaged
excitation is bounded away from zero over every window
of length $T$, which by standard arguments~\cite{Ioannou1996}
implies that $\lambda_{\min}(\bar{\mathbf{Y}}_i^\top
\bar{\mathbf{Y}}_i(t))$ is positive on a set of positive
measure, ensuring $\beta_i > 0$ in a time-averaged sense
sufficient for~(\ref*{eq:Vsdot_bound}).
Then the total augmented Lyapunov function
$\mathcal{V}^a = \sum_{i=1}^{n}\nu_i^a$ under the
adaptive control laws~(\ref*{eq:adaptcont})
and~(\ref*{eq:adaptlaw}) satisfies

\begin{align}
\mathcal{V}^a(t) \leq \mathcal{V}^a(0)e^{-\mu t}
+ \frac{c_Q}{\mu}\left(1 - e^{-\mu t}\right),
\label{eq:Va_bound}
\end{align}

where $\mu = \min_i \alpha_i' > 0$ and
$c_Q = \sum_{i=1}^n c_{Q,i} \geq 0$.
The bound~(\ref*{eq:Va_bound}) establishes exponential
ultimate boundedness of the full adaptive closed-loop
system, with the residual set determined by $c_Q/\mu$.
Since $c_{Q,i} \to 0$ as $\mathbf{e}_{s_i} \to \mathbf{0}$
by~(\ref*{eq:cQi_def}), exponential convergence of both
$\mathbf{e}_{\mathbf{V}_i}$ and $\mathbf{e}_{s_i}$ to
zero is recovered as adaptation completes.

\textit{Proof.}
Summing~(\ref*{eq:nudot_adapt_final2}) over all links:

\begin{align}
\frac{d}{dt}\mathcal{V}^a
\leq -\sum_{i=1}^{n}\alpha_i'\,\nu_i^a
+ c_Q + \sum_{i=1}^{n} p_i.
\label{eq:totalVadot}
\end{align}

The adaptive control law~(\ref*{eq:adaptcont}) modifies
the applied wrench $\boldsymbol{\mathcal{W}}_i$ relative
to the nominal law, but the interaction wrenches
$\boldsymbol{\mathcal{W}}_{J,i}$ at the joint interfaces
are determined by the holonomic constraint forces, not
directly by $\boldsymbol{\mathcal{W}}_i$. The wrench
transformation~(\ref*{eq:wrench_transform2}), the
point-shift~(\ref*{eq:X_wrench}), and the projected
twist error pairing~(\ref*{eq:twist_transform2}) are
therefore unchanged, and the cancellation argument of
Section~\ref{subsec:4d} applies without modification,
giving $\sum_{i=1}^{n} p_i = 0$.
Substituting and setting $\mu = \min_i\alpha_i' > 0$

\begin{align}
\frac{d}{dt}\mathcal{V}^a \leq -\mu\,\mathcal{V}^a + c_Q.
\label{eq:totalVadot2}
\end{align}

By the comparison lemma~\cite{khalil2002},
$\mathcal{V}^a(t) \leq \mathcal{V}^a(0)e^{-\mu t}
+ \frac{c_Q}{\mu}(1 - e^{-\mu t})$,
establishing~(\ref*{eq:Va_bound}). Since $\mathbf{M}_i \succ 0$ and $\boldsymbol{\Lambda}_i \succ 0$,
each subsystem Lyapunov function satisfies
$\nu_i^a \leq \mathcal{V}^a$, so the individual errors
are bounded by

\begin{align}
\|\mathbf{e}_{\mathbf{V}_i}(t)\|
&\leq \sqrt{\frac{\mathcal{V}^a(t)}
{\lambda_{\min}(\mathbf{M}_i)}},
\label{eq:adaptexpstab_V} \\[4pt]
\|\mathbf{e}_{s_i}(t)\|
&\leq \sqrt{\lambda_{\max}(\boldsymbol{\Lambda}_i)
\,\mathcal{V}^a(t)},
\quad i = 1,\ldots,n,
\label{eq:adaptexpstab_s}
\end{align}

which are bounded for all $t \geq 0$ and decay to the
residual set of radius $\sqrt{c_Q/(\mu\lambda_{\min}(\mathbf{M}_i))}$
and $\sqrt{\lambda_{\max}(\boldsymbol{\Lambda}_i)c_Q/\mu}$
respectively, determined by the model mismatch
term~(\ref*{eq:cQi_def}). As $c_{Q,i} \to 0$ with
parameter convergence, both residual sets shrink to zero
and exponential convergence at rate $\mu/2$ is recovered.

\begin{remark}
The PE condition~(\ref*{eq:PE}) is a property of the
closed-loop trajectory and cannot be enforced by
controller design. It is satisfied when the reference
trajectory generates sufficient joint-space motion to
excite all parametric directions of $\bar{\mathbf{Y}}_i$
over every window of length $T$. In practice,
PE is verified offline by numerically
evaluating~(\ref*{eq:PE}) along the simulated trajectory
and checking that the smallest eigenvalue of the
windowed Gramian $\int_t^{t+T}
\bar{\mathbf{Y}}_i^\top\bar{\mathbf{Y}}_i\,d\tau$
remains bounded away from zero. 
\end{remark}

The implementation procedure for the presented
Subsystem-based Lie Algebraic PDE Control (SLPC)
scheme is described in Table~\ref{tab:1}.

\begin{table}[!t]
\renewcommand{\arraystretch}{1.1}
\setlength{\arrayrulewidth}{0.8pt}
\setlength{\tabcolsep}{4pt}
\centering
\begin{tabular}{|p{7.3cm}|}
\hline
\textbf{// Initialization} \\
\hline
\textbf{i.} Input: mechanical and geometric properties
of each link $i = 1,\ldots,n$ \\
\textbf{ii.} Input: desired endpoint trajectory
${}^{o}\mathbf{p}_{e,d}(t)$,
${}^{o}\dot{\mathbf{p}}_{e,d}(t)$ \\
\textbf{iii.} Input: initial body-fixed states
${}^{i}\mathbf{V}_i(0)$,
$\mathbf{r}_{\xi_i}(\xi,0)$,
$\dot{\mathbf{r}}_{\xi_i}(\xi,0)$,
$\mathbf{Ad}_{oi}(0)$ \\
\textbf{iv.} Set: control gains $\mathbf{K}_i \succ 0$,
adaptation gains $\boldsymbol{\Lambda}_i \succ 0$;
uncertainty bounds $\mathbf{s}_{i,l}$, $\mathbf{s}_{i,h}$;
initial estimates $\hat{\mathbf{s}}_i(0) \in \mathcal{Q}_i$;
compute $c_{M,i}$, $c_{H,i}$ offline
via~(\ref*{eq:cMcH}) \\
\textbf{v.} Initialize: $\hat{\mathbf{M}}_i(0)$,
$\hat{\mathbf{H}}_{c,i}(0)$ from $\hat{\mathbf{s}}_i(0)$ \\
\hline
\textbf{// Sample $k$ ($k \geq 1$)} \\
\hline
\textbf{vi.} \textit{State measurement:} \\
\quad read: $\mathbf{V}_i(k)$ from joint encoders/IMUs \\
\quad read: $\mathbf{r}_{\xi_i}(c_i,k)$,
$\dot{\mathbf{r}}_{\xi_i}(c_i,k)$ from tip sensors \\
\quad update: $\mathbf{Ad}_{oi}(k)$
from ${}^{i}\mathbf{q}_i(k)$ \\
\hline
\textbf{vii.} \textit{Reference generation
(\S\ref{subsec:4a}):} \\
\quad compute: $\boldsymbol{\delta}_e^{\mathrm{est}}(k)$
via~(\ref*{eq:delta_est}),
${}^{o}\mathbf{p}_{e,d}^{\mathrm{corr}}(k)$
via~(\ref*{eq:p_corr}) \\
\quad solve: IK~(\ref*{eq:qjd})
$\to \mathbf{q}_{j,d}(k)$, $\dot{\mathbf{q}}_{j,d}(k)$ \\
\quad assemble: $\mathbf{V}_{d,i}(k)$
via~(\ref*{eq:Vdi}),
$\dot{\mathbf{V}}_{d,i}(k)$ by differentiation \\
\hline
\textbf{viii.} \textit{Twist error:}
$\mathbf{e}_{\mathbf{V}_i}(k)
= \mathbf{V}_{d,i}(k) - \mathbf{V}_i(k)$ \\
\hline
\textbf{ix.} \textit{Control input~(\ref*{eq:adaptcont}):} \\
\quad evaluate: $\hat{\mathbf{M}}_i(k)$,
$\hat{\mathbf{H}}_{c,i}(k)$
from $\hat{\mathbf{s}}_i(k)$ \\
\quad compute and apply:
$\boldsymbol{\mathcal{W}}_i(k) =
\hat{\mathbf{M}}_i\dot{\mathbf{V}}_{d,i}
+ \hat{\mathbf{H}}_{c,i}
+ \mathbf{K}_i\hat{\mathbf{M}}_i
\mathbf{e}_{\mathbf{V}_i}$
via~(\ref*{eq:30}) \\
\hline
\textbf{x.} \textit{Parallel model and residuals:} \\
\quad evaluate parallel model at $\hat{\mathbf{s}}_i(k)$
using~(\ref*{eq:contform}), (\ref*{eq:strainpde}) \\
\quad compute: $\boldsymbol{\epsilon}_{V,i}(k)$
via~(\ref*{eq:residual_V}),
$\boldsymbol{\epsilon}_{\xi,i}(k)$
via~(\ref*{eq:residual_xi}) \\
\quad assemble: $\boldsymbol{\Gamma}_i(k)$
via~(\ref*{eq:lumped_signal}) \\
\hline
\textbf{xi.} \textit{Parameter update~(\ref*{eq:adaptlaw}):} \\
\quad $\hat{\mathbf{s}}_i(k{+}1) =
\hat{\mathbf{s}}_i(k) + T_s\,\mathcal{P}\!\left(
\hat{\mathbf{s}}_i,\mathbf{s}_{i,l},\mathbf{s}_{i,h},
\boldsymbol{\Lambda}_i\boldsymbol{\Gamma}_i\right)$ \\
\hline
\textbf{xii.} $k = k+1$, goto \textbf{(vi)} \\
\hline
\end{tabular}
\caption{\textbf{SLPC implementation algorithm.}
$T_s$ denotes the sampling period.}
\label{tab:1}
\end{table}

\section{Results and discussion} \label{sec:6}

The results below demonstrate three properties of the
proposed framework on a two-link rigid--flexible manipulator
executing a spatial circular trajectory: closed-loop
implementability under the subsystem-based adaptive
controller, consistency with the theoretical stability
and boundedness results of
Sections~\ref{sec:4}--\ref{sec:5}, and the practical
advantage of the model-based and twist-based formulations
over conventional joint-space control. Three control strategies are compared: the proposed
subsystem-based Lie algebraic PDE controller (SLPC),
a proportional twist-based controller (PTC) that retains the
$\mathfrak{se}(3)$ reference structure without
model-based feedforward, and a naive joint-space PD
controller (PD). The three-way comparison isolates the
respective contributions of the model-based dynamic
cancellation, the twist-based kinematic reference
generator, and the Lie-algebraic formulation --- each
adding measurable independent performance benefit.

\subsection*{Illustrative simulation study}
\label{subsec:sim}

The simulated system consists of a rigid first link and a
flexible second link connected in series, operating in
three-dimensional space. The mechanical properties are:
Young's modulus $E = 2.1\times10^{11}$~[Pa], densities
$\rho_1 = \rho_2 = 7800$~[kg/m$^3$], cross-sections
$b_1\times h_1 = 10\times30$~[mm] and
$b_2\times h_2 = 10\times50$~[mm], and lengths
$l_1 = 1.2$~[m] and $l_2 = 1.0$~[m].

The two links are connected by two revolute joints. The
base joint actuates link~1 about its $y$- and $z$-axes
($SO(2)\times SO(2)$); the connecting joint actuates
link~2 about its body-fixed $z$-axis ($SO(2)$). The
joint constraint projection
matrices~(\ref*{eq:joint_constraint2}) are

\begin{align}
\mathbf{I}_3 - {}^{o}\mathbf{A}_1\,{}^{o}\mathbf{A}_1^\top
&= \mathrm{diag}(0,\,1,\,1), \notag \\
\mathbf{I}_3 - {}^{o}\mathbf{A}_2\,{}^{o}\mathbf{A}_2^\top
&= \mathrm{diag}(0,\,0,\,1),
\label{eq:projmat}
\end{align}

confirming link~1 is free to rotate about $y$- and
$z$-axes while link~2 is constrained to its $z$-axis,
with $x$-components of both twist errors projected to
zero. The simulation is initialised with non-zero joint
angles ($\theta_{1z}(0) = \pi/6$,
$\theta_{2z}(0) = \theta_{1z}(0) + \pi/8$) and zero
elastic modal coordinates.

The desired endpoint trajectory is a circular path of
radius $r_d = 0.5$~[m] in the $yz$-plane

\begin{align}
{}^0p_{e,d,y}(t) &= r_d \sin(\omega_d t)\,\rho(t), \notag \\
{}^0p_{e,d,z}(t) &= r_d \cos(\omega_d t)\,\rho(t),
\label{eq:traj_circle}
\end{align}

where $\omega_d = 1.0$~[rad/s] and
$\rho(t) = 1 - e^{-t/\tau_{\mathrm{ramp}}}$ with
$\tau_{\mathrm{ramp}} = 1.5$~[s]. The desired joint
angles from the deflection-corrected inverse
kinematics~(\ref*{eq:p_corr})--(\ref*{eq:qjd}) are

\begin{align}
\theta_{1z,d} &= \frac{{}^0p_{e,d,y}}{L}
+ \theta_{1z}(0)\,b(t), \notag \\
\theta_{1y,d} &= -\frac{{}^0p_{e,d,z}}{L}
+ \theta_{1y}(0)\,b(t), \notag \\
\theta_{2z,d} &= \frac{{}^0p_{e,d,y}}{L}
+ \theta_{2z}(0)\,b(t),
\label{eq:traj_ik}
\end{align}

where $L = l_1 + l_2$ and $b(t) = e^{-t/\tau_{\mathrm{blend}}}$
with $\tau_{\mathrm{blend}} = 2.0$~[s] smoothly drives
initial offsets to zero without discontinuities in the
control input or twist reference. Desired body-fixed
twists $\mathbf{V}_{d,i}$ and $\dot{\mathbf{V}}_{d,i}$
are assembled via~(\ref*{eq:Vdi}) and numerical
differentiation at each sample. The simulation runs for
$t_f = 25$~[s] at $T = 10^{-3}$~[s] per
Table~\ref{tab:1}; modal projection and integration
details follow~\cite{Yaqubi2026}.

The PTC control law is

\begin{align}
\boldsymbol{\mathcal{W}}_i^{\mathrm{PTC}}
= \mathbf{K}_i\!\left(\mathbf{V}_{d,i} - \mathbf{V}_i\right),
\label{eq:PTC}
\end{align}

retaining the $\mathfrak{se}(3)$ reference structure
and inter-link kinematic propagation of
Section~\ref{subsec:4a} without model-based feedforward.
The PD control law

\begin{align}
\boldsymbol{\tau}_i^{\mathrm{PD}}
= \mathbf{K}_{p,i}\!\left(\boldsymbol{\theta}_{d,i}
- \boldsymbol{\theta}_i\right)
+ \mathbf{K}_{d,i}\!\left(\dot{\boldsymbol{\theta}}_{d,i}
- \dot{\boldsymbol{\theta}}_i\right)
\label{eq:pdbasic}
\end{align}

has no twist structure or chain geometry. Control gains
are $\mathbf{K}_1 = \mathrm{diag}(0,0,0,0,300,300)$
and $\mathbf{K}_2 = \mathrm{diag}(0,0,0,0,0,200)$.
Motor inertia of $[1.0,\,3.0,\,1.7]$~[kg$\cdot$m$^2$]
is included at the joints. PD gains are $K_p = 500$,
$K_d = 20$ for link~1 and $K_p = 400$, $K_d = 20$
for link~2. Torque saturation of $\pm100$~[Nm] is
applied on all inputs.

\subsubsection*{Nominal controller}

The nominal controller~(\ref*{eq:nomcont}) cancels the
inertial, Coriolis, and elastic restoring terms through
$\mathbf{M}_i\dot{\mathbf{V}}_{d,i} + \mathbf{H}_{c,i}$,
reducing the closed-loop error dynamics to the dissipative
form~(\ref*{eq:errdyn}) with exponential convergence
guaranteed by Theorem~1 independently of system topology.

\textit{Twist response and joint tracking.}
The body-fixed twist components and joint angle tracking
responses are shown in Fig.~\ref{fig:twists_tracking}.
The translational velocity $\mathbf{v}_1 \equiv \mathbf{0}$
identically since the base joint does not translate, while
$\boldsymbol{\omega}_1$ is dominated by the $y$- and
$z$-components of the two active trajectory axes with
$x$-component constrained to zero by~(\ref*{eq:projmat}).
Link~2 inherits a full six-dimensional twist:
$\boldsymbol{\omega}_2$ from chain propagation through
the Adjoint map and the relative $z$-rotation, while
$\mathbf{v}_2$ arises from link~1 rotation through
offset $\mathbf{r}_{20}$~\cite{Yaqubi20262}. All twist
components remain bounded throughout the $25$~[s]
horizon, confirming the boundedness result of
Section~\ref{subsec:4d}. Joint angle tracking under the
SLPC follows the desired references with the blend
function~(\ref*{eq:traj_ik}) successfully absorbing
the initial configuration offset. The PTC achieves
comparable steady-state tracking but exhibits larger
transient error during ramp-up where feedforward absence
produces a phase lag. The naive PD produces substantially
larger errors throughout, confirming that twist-based
reference generation is a significant independent
contributor to tracking performance. SLPC stability is
guaranteed by Theorem~2 for any positive gain matrix,
whereas PTC and PD stability depends significantly on
gain tuning.

\textit{Control inputs and endpoint trajectory.}
The torque inputs and three-dimensional endpoint
trajectory are shown in Fig.~\ref{fig:inputs_traj}.
The SLPC traces the desired circular path closely,
overlapping the reference throughout steady-state.
The PTC produces a small geometric deviation most
apparent near trajectory extrema where inertial
loading is greatest; the PD produces a substantially
larger deviation. An initial torque spike at
$t \approx 0$ arises from the non-zero initial
configuration offset, most pronounced in the PD
which reacts through feedback alone, within the
$\pm100$~[Nm] saturation bound for all controllers.
Beyond the transient, the SLPC produces the smoothest
steady-state torque profile as feedforward cancellation
reduces required feedback effort; the PTC is comparable;
the PD is more oscillatory.

\subsubsection*{Adaptive controller}

The present implementation identifies the
$5\times 2 = 10$ dominant parametric effects across
both links:

\begin{align}
\hat{\mathbf{s}}_i =
\begin{bmatrix}
\hat{\rho}_i A_i,\;
\widehat{[\mathbf{I}_{b_i}]_{22}},\;
\widehat{[\mathbf{I}_{b_i}]_{33}},\;
\widehat{EI_{y_i}},\;
\widehat{EI_{z_i}}
\end{bmatrix}^\top \in \mathbb{R}^5,
\end{align}

where $\rho_i A_i$ governs distributed inertial terms;
$[\mathbf{I}_{b_i}]_{22}$, $[\mathbf{I}_{b_i}]_{33}$
are the nonzero independent rotational inertia entries
for a symmetric cross-section; and $EI_{y_i}$,
$EI_{z_i}$ are the principal bending stiffnesses.
Axial stiffness $EA_i$ is treated as known since its
contribution is of order $(b/l)^2 \approx 10^{-4}$
relative to bending. For the symmetric cross-section,
$[\mathbf{I}_{b_i}]_{22} = [\mathbf{I}_{b_i}]_{33}
=: I_{b,i}^{\mathrm{eff}}$ by symmetry and both
entries receive identical adaptation signals; they are
retained separately for generality. The PE condition
is verified on the reduced $4$-dimensional identifiable
subspace $\{\rho_i A_i,\,I_{b,i}^{\mathrm{eff}},\,
EI_{y_i},\,EI_{z_i}\}$ following
Remark~\ref{rem:si_reduced}. Full parametric
identification can be incorporated by extending
$\hat{\mathbf{s}}_i$ without modifying the controller
structure.

Estimates are initialised within $\pm10\%$ of true
values with $\pm20\%$ projection bounds enforced
by~(\ref*{eq:proj}). The adaptation signal
$\boldsymbol{\Gamma}_i$~(\ref*{eq:lumped_signal}) is
assembled from $\boldsymbol{\epsilon}_{V,i}$ and
$\boldsymbol{\epsilon}_{\xi,i}$ computed at each sample
from~(\ref*{eq:contform}) and~(\ref*{eq:strainpde})
at the link tip. The diagonal gain matrix
$\boldsymbol{\Lambda}_i =
\mathrm{diag}(\lambda_{\rho,i},\,\lambda_{b22,i},\,
\lambda_{b33,i},\,\lambda_{EIy,i},\,\lambda_{EIz,i})$
takes values $\lambda_{\rho,i} = 5\times10^5$,
$\lambda_{b22,i} = \lambda_{b33,i} = 10^3$,
$\lambda_{EIy,i} = 10$, $\lambda_{EIz,i} = 100$,
reflecting differing channel sensitivities: larger
$\lambda_{\rho,i}$ compensates for the weaker
translational acceleration signal, while the asymmetry
between $\lambda_{EIy,i}$ and $\lambda_{EIz,i}$
reflects stronger $z$-plane excitation by the circular
trajectory. The $\boldsymbol{\epsilon}_{V,i}$ residual
updates $\rho_i A_i$ and $[\mathbf{I}_{b_i}]_{22,33}$
via $\mathbf{Y}_{V,i}$~(\ref*{eq:regressor_V});
$\boldsymbol{\epsilon}_{\xi,i}$ updates $EI_{y_i}$ and
$EI_{z_i}$ via
$\mathbf{Y}_{\xi,i}$~(\ref*{eq:regressor_xi}).

Figure~\ref{fig:adaptation}(a) shows convergence of
all ten parameters within the $25$~[s] simulation.
Inertia parameters converge in $3$--$5$~[s]; $EI_{z_i}$
converges faster than $EI_{y_i}$ due to stronger
$z$-plane excitation; $\rho_i A_i$ converges more
slowly but reaches the true value within the simulation
horizon. All estimates remain within $\mathcal{Q}_i$
by Lemma~\ref{lem:proj}. Normalised errors reduce from
$10\%$ to below $2\%$ within $5$~[s] for inertia and
stiffness and by $10$~[s] for density, while tracking
performance remains essentially identical to the nominal
case, confirming Theorem~4 and the independence of
stability and adaptation loops established in
Section~\ref{subsec:5e}.

The PE condition~(\ref*{eq:PE}) is verified numerically
in Fig.~\ref{fig:adaptation}(b), which plots
$\lambda_{\min}$ of the windowed Gramian
$\int_t^{t+T}\bar{\mathbf{Y}}_i^\top
\bar{\mathbf{Y}}_i\,d\tau$ over the window
$T = 2\pi/\omega_d$. The eigenvalue remains strictly
positive throughout the active adaptation phase,
confirming PE is satisfied. Displaying the full $25$~[s]
horizon would show only zero after parameter convergence
and obscure the PE verification during the active
identification window. The subsequent decay is
consistent with parameter convergence: as
$\hat{\mathbf{s}}_i \to \mathbf{s}_i$, the residuals
$\boldsymbol{\epsilon}_{V,i},\boldsymbol{\epsilon}_{\xi,i}
\to \mathbf{0}$ by~(\ref*{eq:residual_V})--(\ref*{eq:residual_xi}),
giving $\boldsymbol{\Gamma}_i \to \mathbf{0}$
by~(\ref*{eq:lumped_signal}) --- the expected behaviour
of a correctly converging adaptive system.

\textit{Deformation field.}
The distributed deformation fields are shown in
Fig.~\ref{fig:deformation}. The deformation is maximum
at the free tip and zero at the clamped base,
consistent with boundary
conditions~(\ref*{eq:35})--(\ref*{eq:38}). Peak tip
deflections of approximately $65$~[mm] in the $y$-
and $z$-directions remain small relative to $l_2 = 1$~[m],
confirming the small-deformation assumption. The axial
deformation is of order $10^{-4}$~[m] --- four orders
of magnitude smaller than bending deflections ---
consistent with the high axial stiffness and confirming
correct exclusion of $EA_i$ from $\hat{\mathbf{s}}_i$.
The full spatiotemporal field is available at every
sample as a direct consequence of the PDE-based
representation, enabling distributed state monitoring
without FEM overhead.

\begin{figure*}[!t]
\centering
\subfloat[]{
\includegraphics[width=0.48\textwidth]{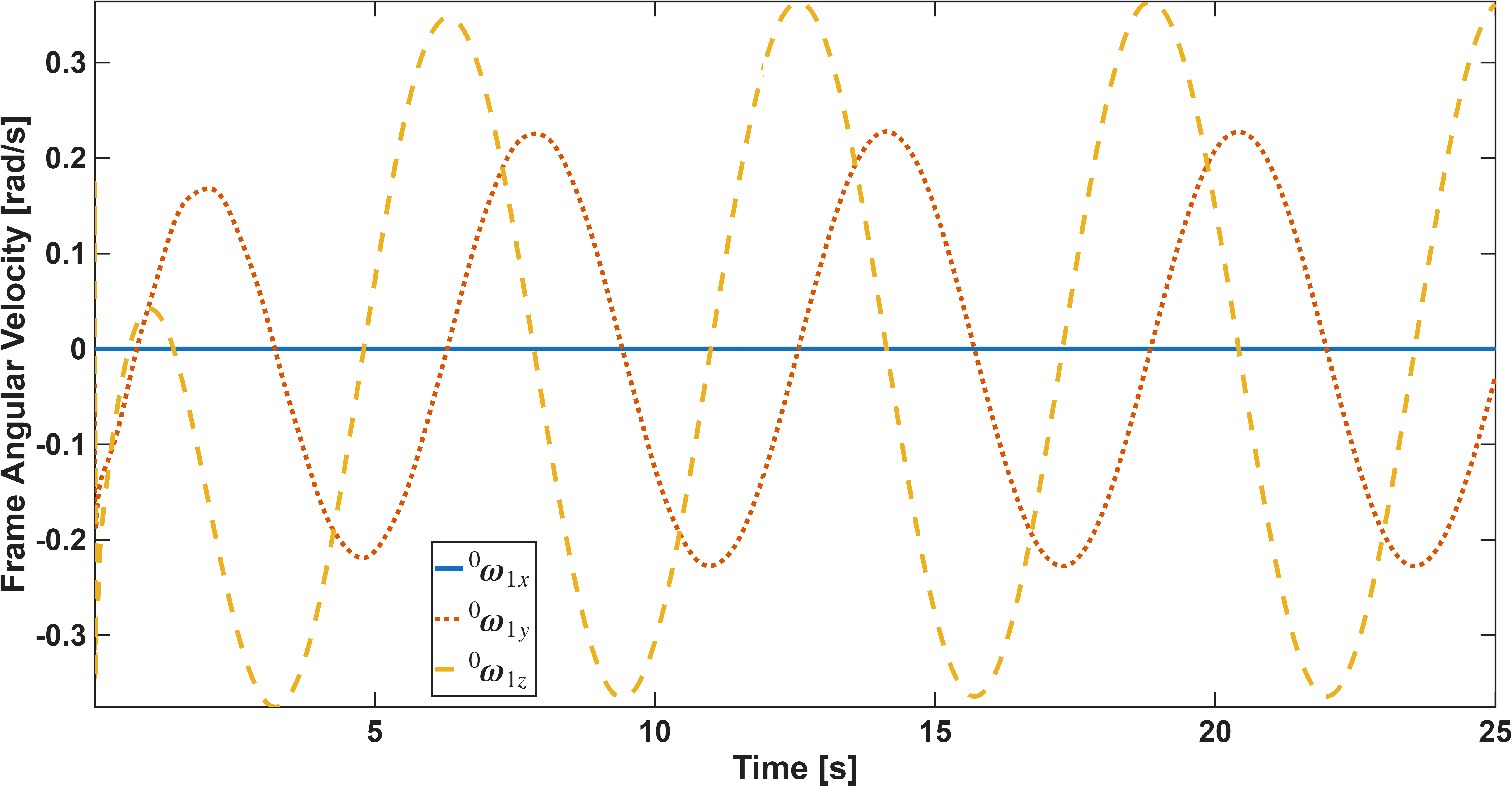}
\label{fig:twist_omega1}}
\hfill
\subfloat[]{
\includegraphics[width=0.48\textwidth]{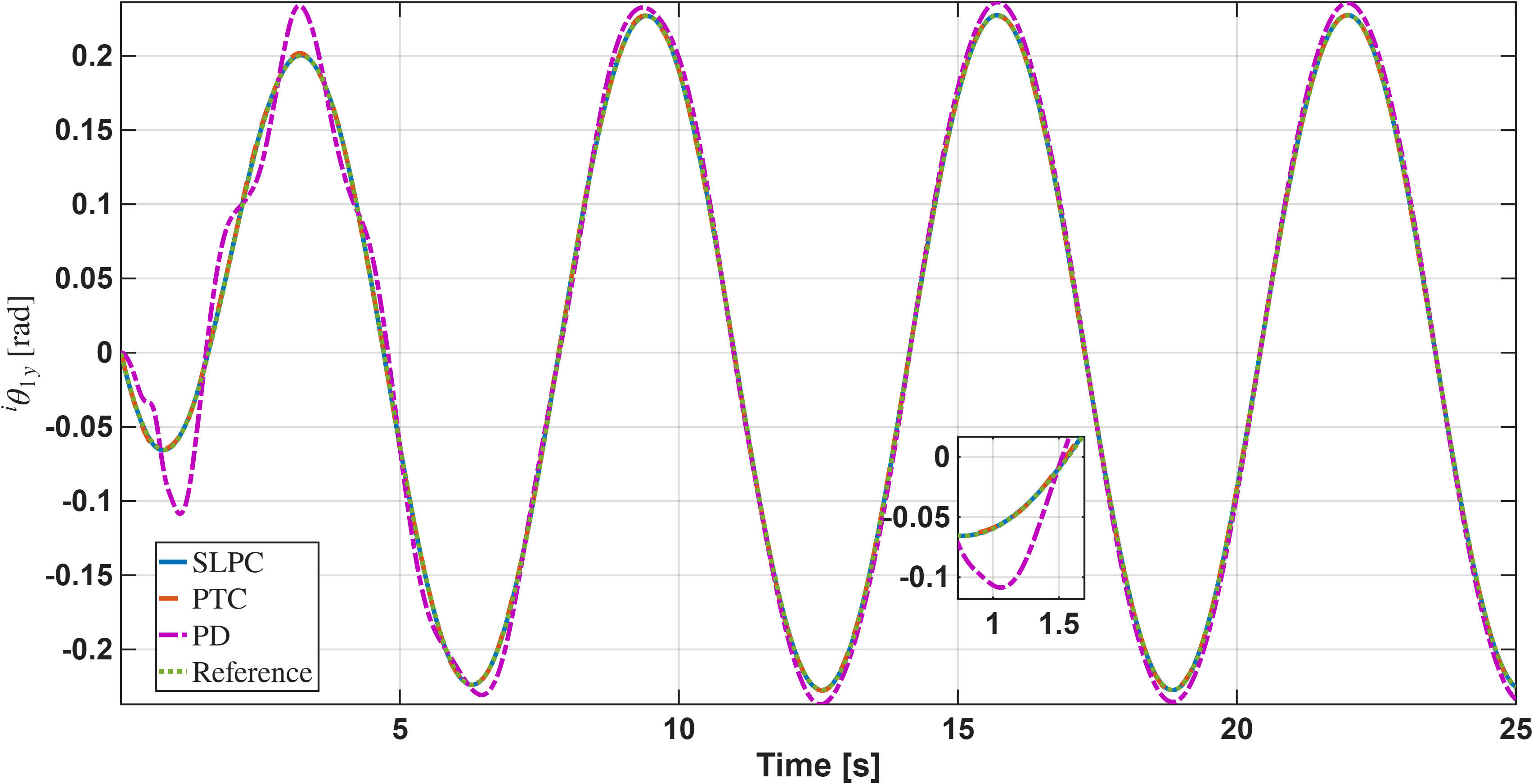}
\label{fig:joint_theta1y}}

\vspace{0.5em}

\subfloat[]{
\includegraphics[width=0.48\textwidth]{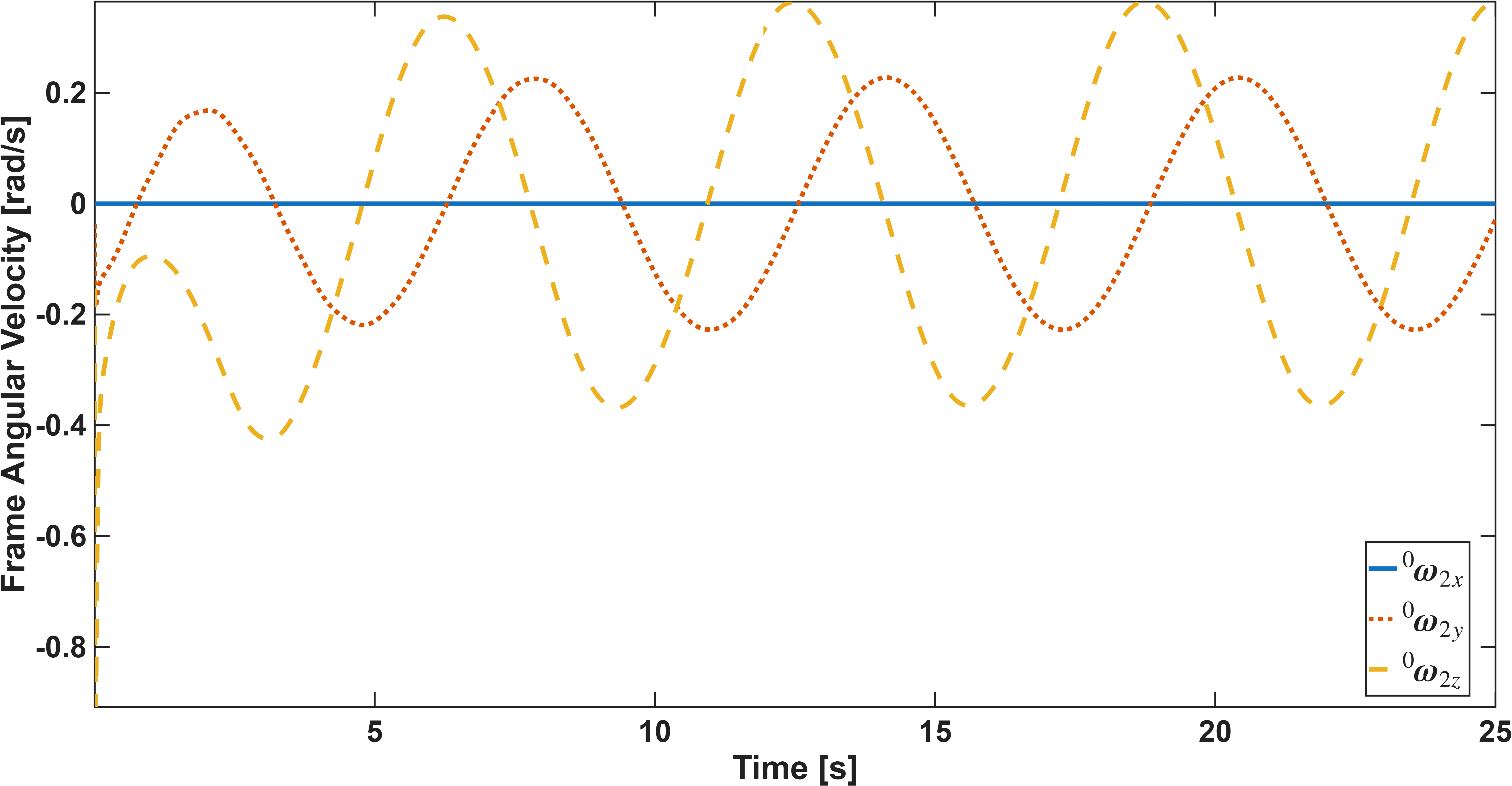}
\label{fig:twist_omega2}}
\hfill
\subfloat[]{
\includegraphics[width=0.48\textwidth]{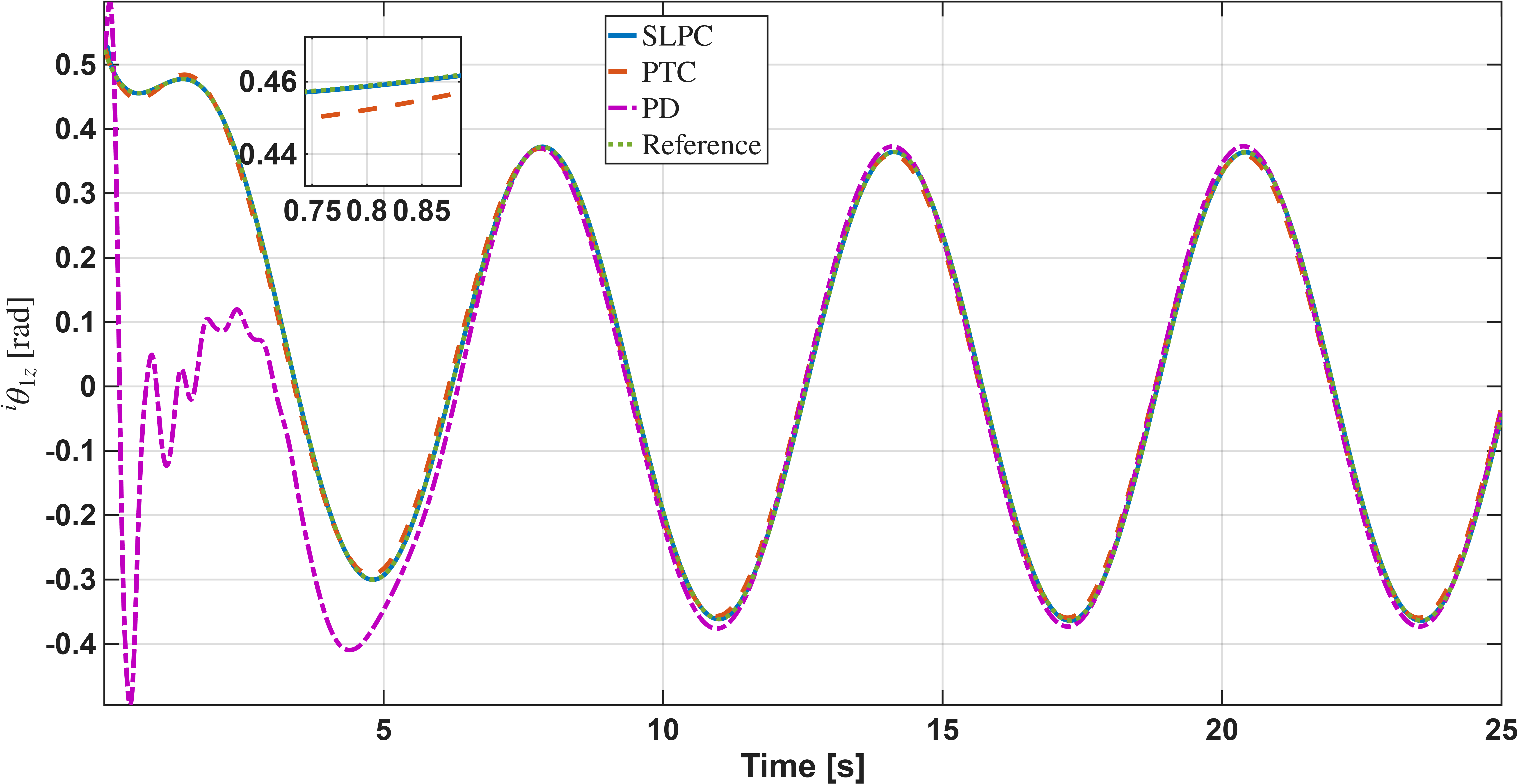}
\label{fig:joint_theta1z}}

\vspace{0.5em}

\subfloat[]{
\includegraphics[width=0.48\textwidth]{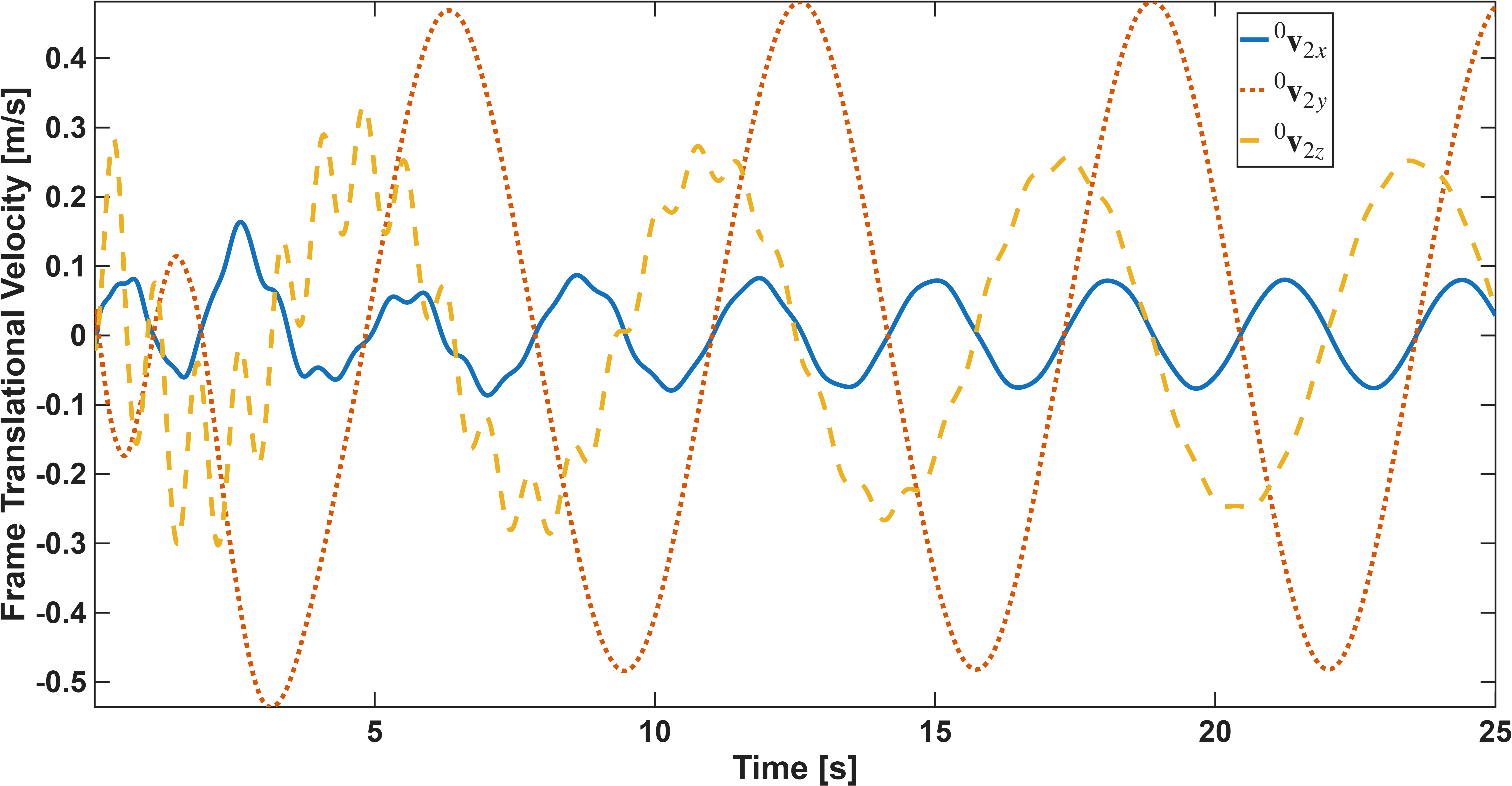}
\label{fig:twist_v2}}
\hfill
\subfloat[]{
\includegraphics[width=0.48\textwidth]{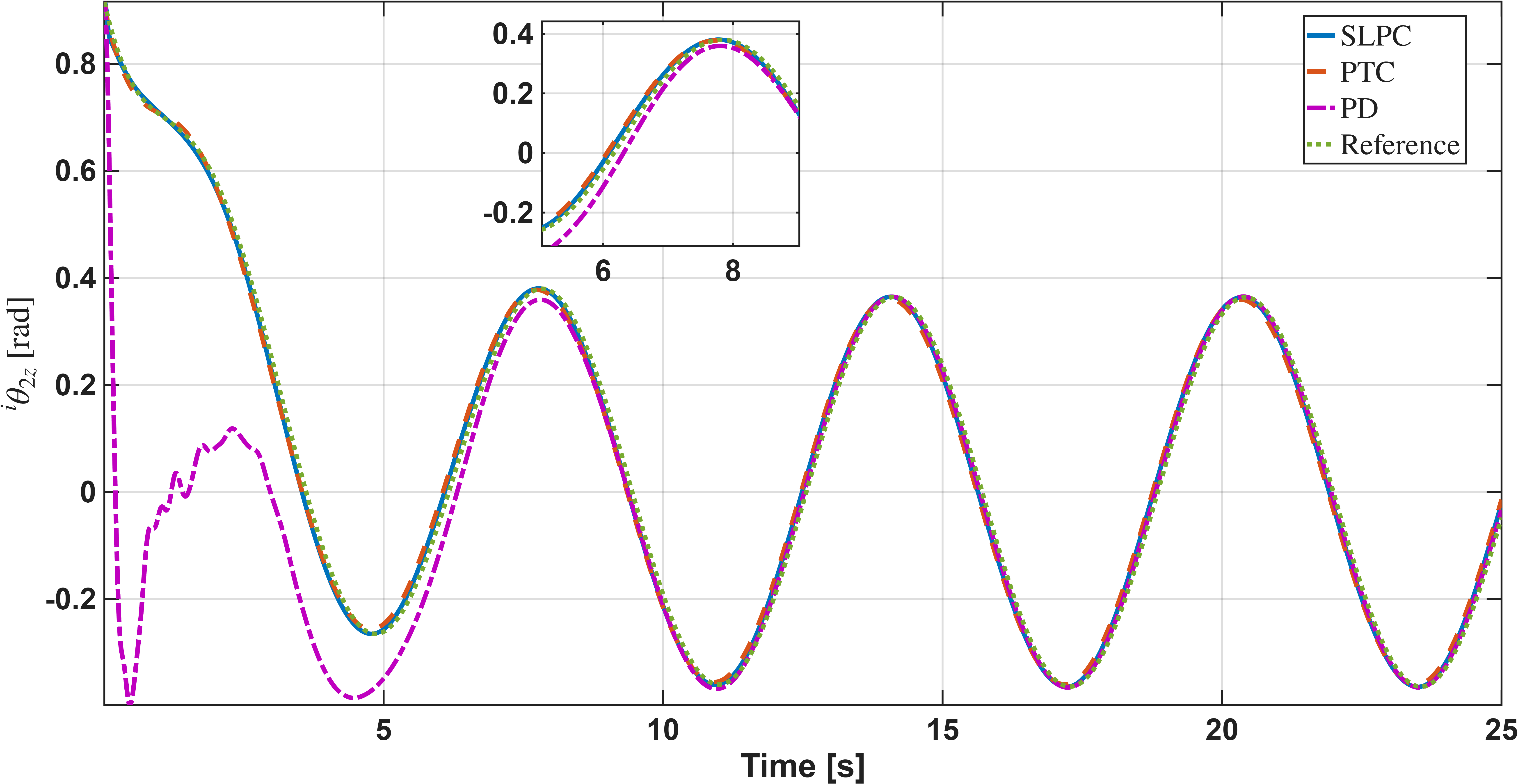}
\label{fig:joint_theta2z}}

\caption{Body-fixed twist components under SLPC and
joint angle tracking under SLPC, PTC, and PD:
(a)~${}^0\boldsymbol{\omega}_1$,
(b)~${}^i\theta_{1y}$,
(c)~${}^0\boldsymbol{\omega}_2$,
(d)~${}^i\theta_{1z}$,
(e)~${}^0\mathbf{v}_2$,
(f)~${}^i\theta_{2z}$.
${}^0\mathbf{v}_1 \equiv \mathbf{0}$ identically.
All twist components bounded throughout the $25$~[s]
horizon. Insets show transient tracking detail.}
\label{fig:twists_tracking}
\end{figure*}

\begin{figure*}[!t]
\centering
\subfloat[]{
\includegraphics[width=0.48\textwidth]{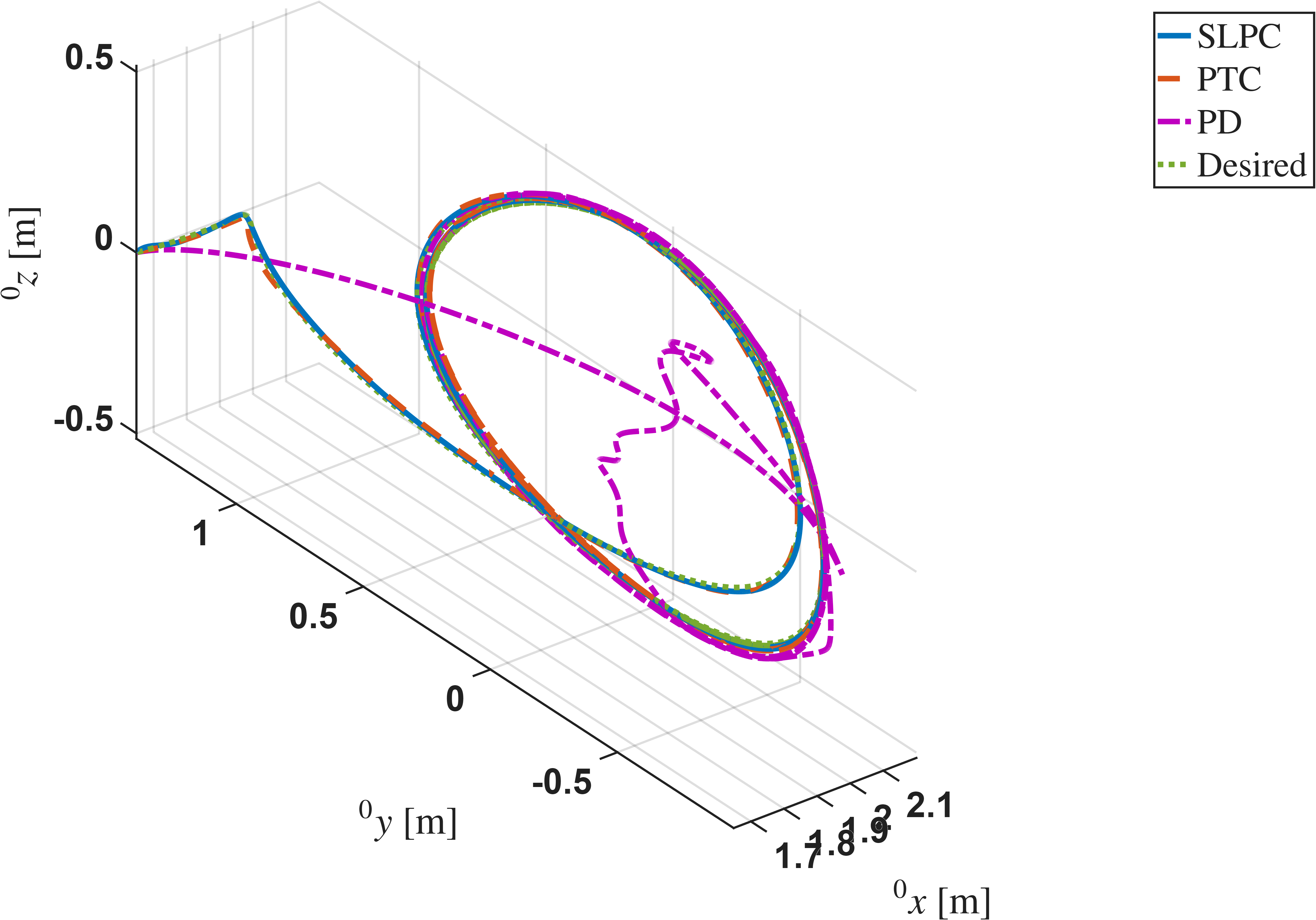}
\label{fig:traj_3d}}
\hfill
\subfloat[]{
\includegraphics[width=0.48\textwidth]{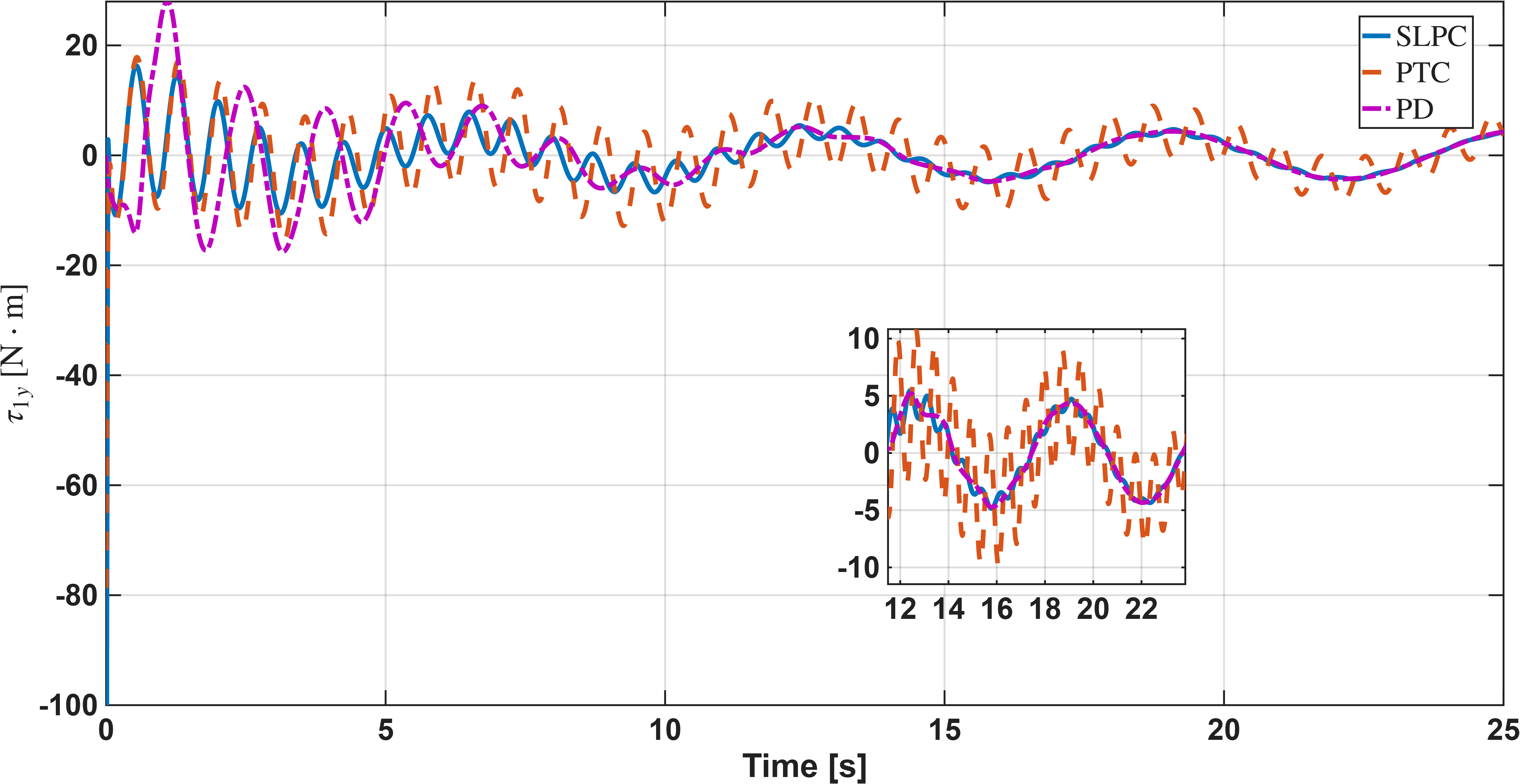}
\label{fig:input_tau1y}}

\vspace{0.5em}

\subfloat[]{
\includegraphics[width=0.48\textwidth]{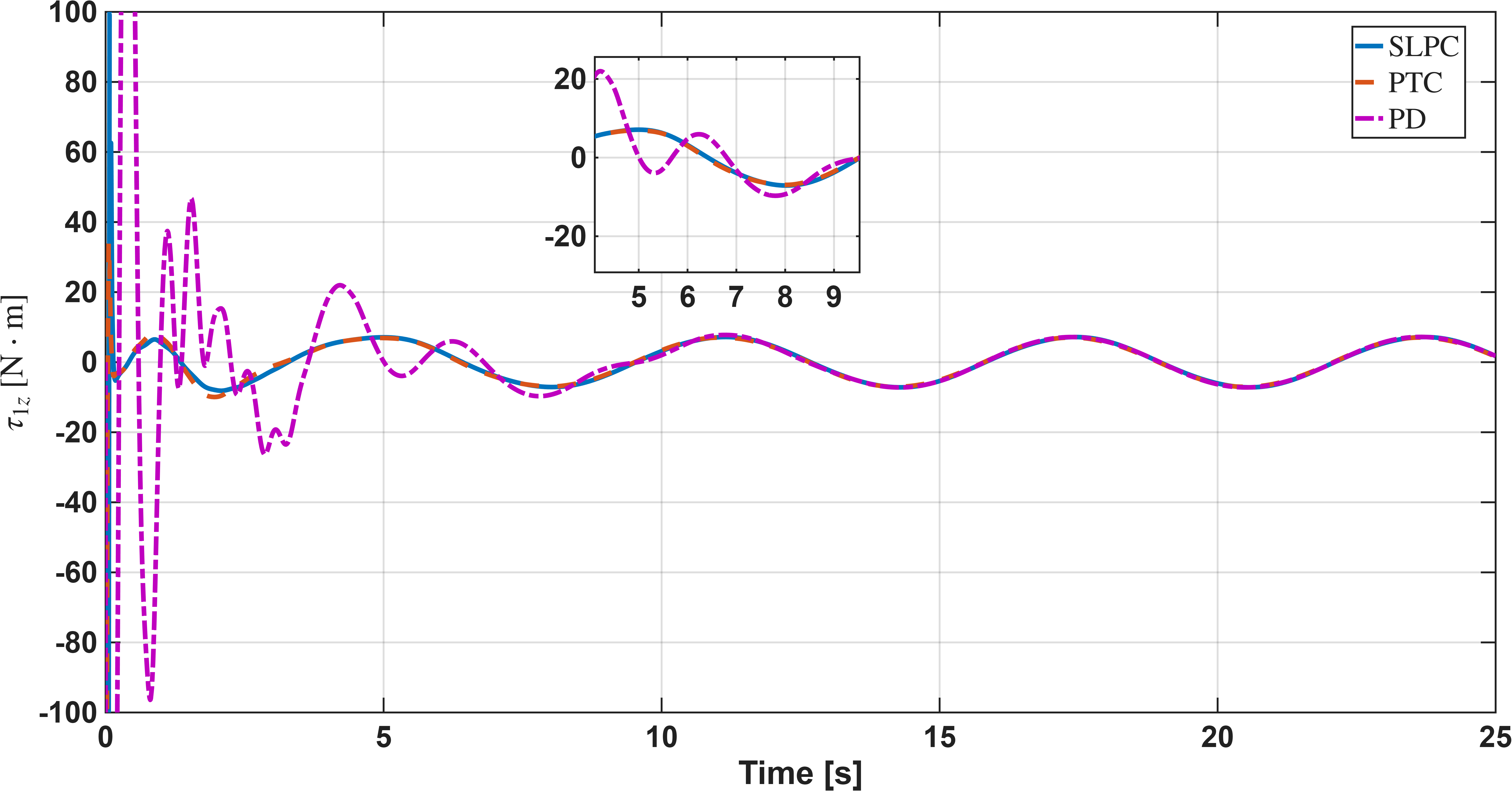}
\label{fig:input_tau1z}}
\hfill
\subfloat[]{
\includegraphics[width=0.48\textwidth]{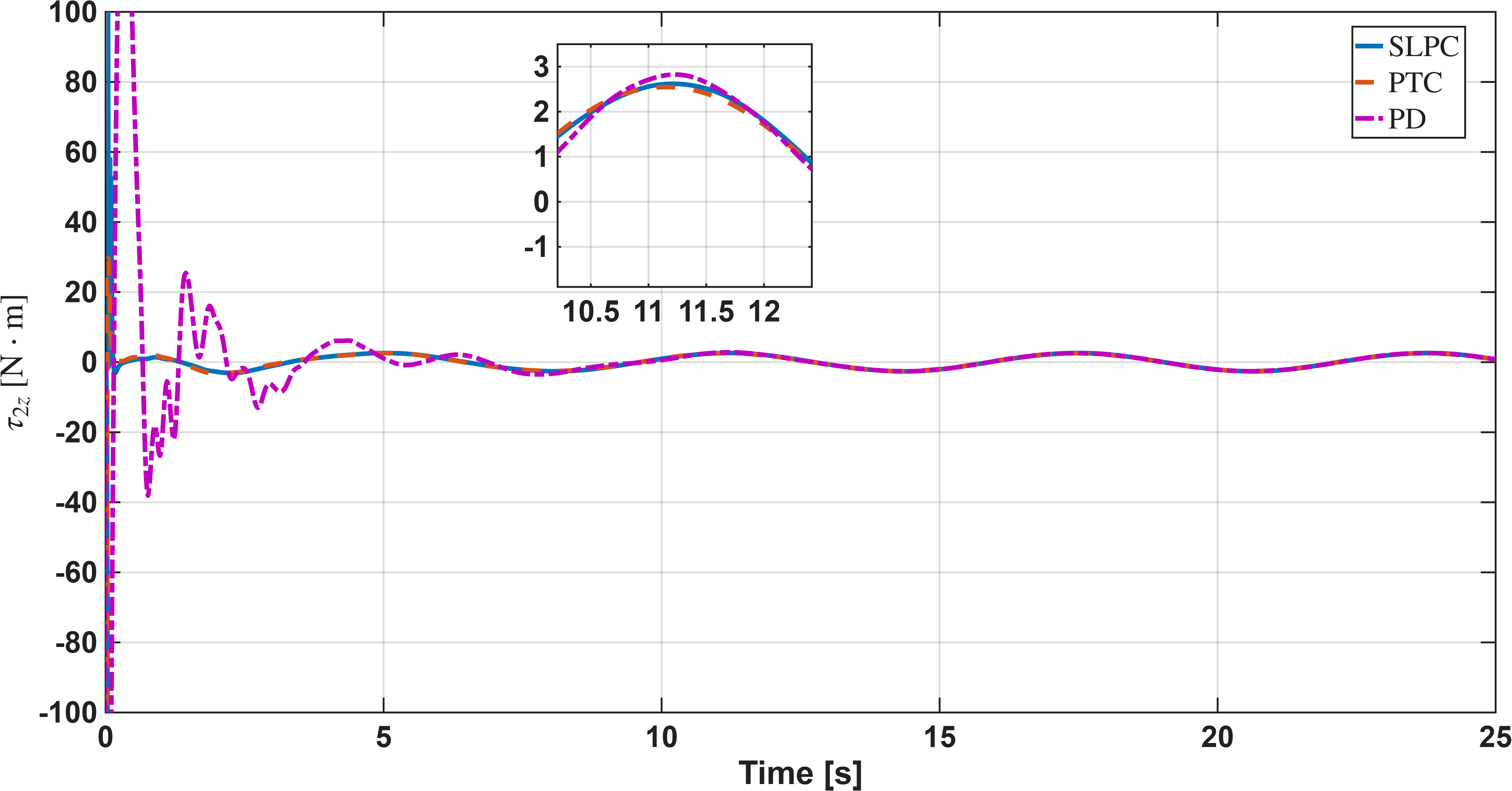}
\label{fig:input_tau2z}}

\caption{Endpoint trajectory and control inputs
comparing SLPC, PTC, and PD with $\pm100$~[Nm]
saturation:
(a)~three-dimensional endpoint trajectory,
(b)~$\tau_{1y}$, (c)~$\tau_{1z}$, (d)~$\tau_{2z}$.}
\label{fig:inputs_traj}
\end{figure*}

\begin{figure*}[!t]
\centering
\subfloat[]{
\includegraphics[width=0.48\textwidth]{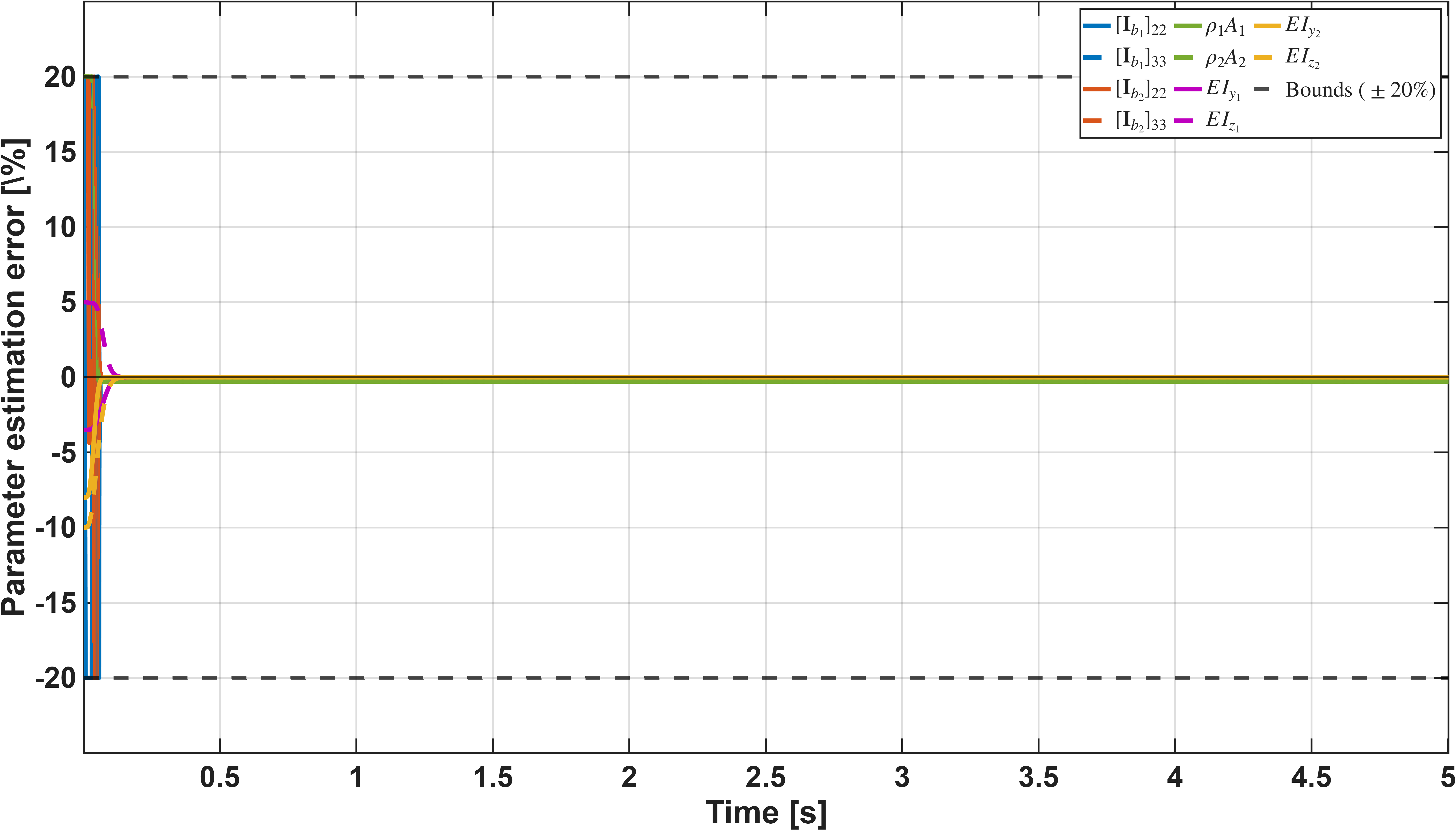}
\label{fig:adapt_params}}
\hfill
\subfloat[]{
\includegraphics[width=0.48\textwidth]{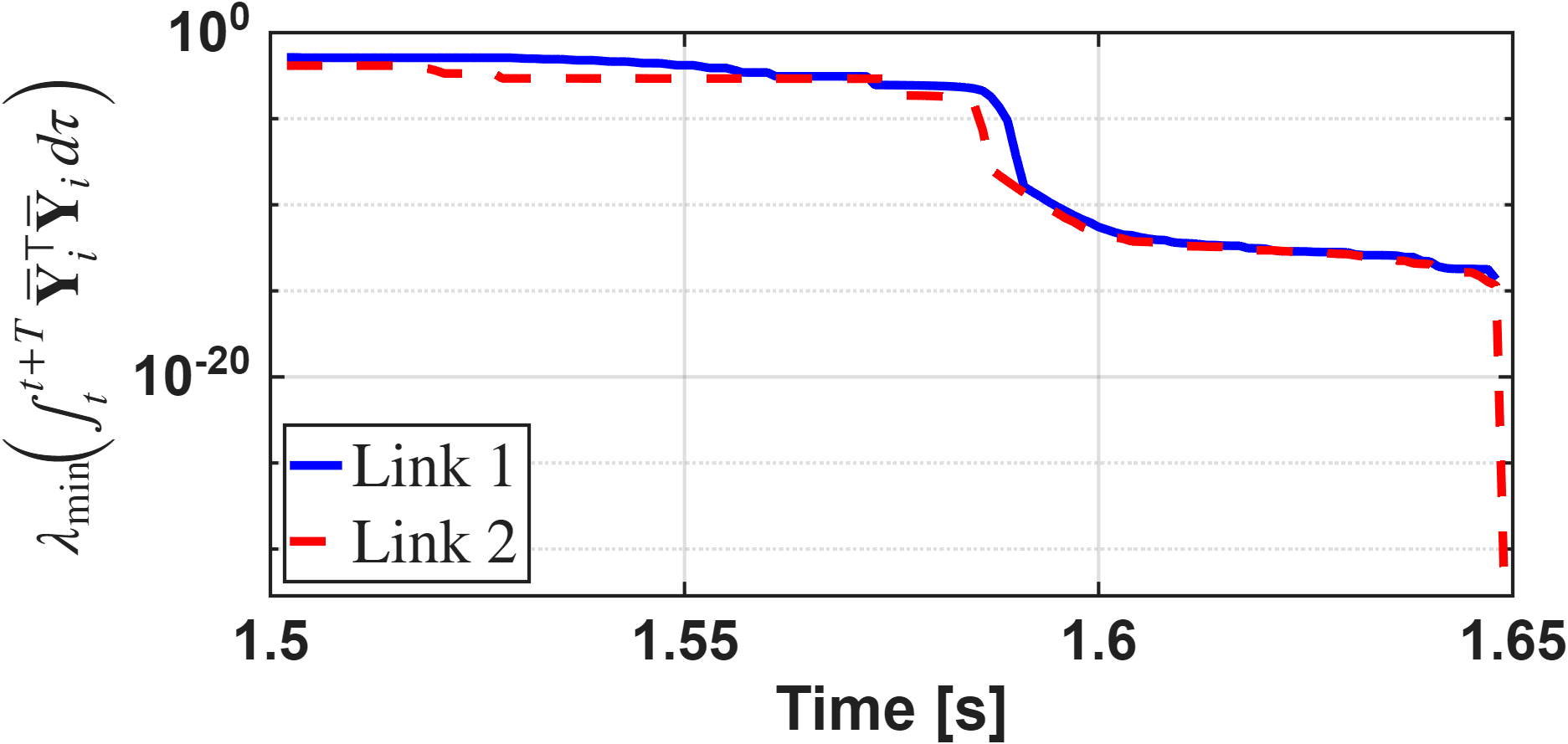}
\label{fig:adapt_PE}}
\caption{Parameter adaptation analysis:
(a)~normalized parameter estimation errors for all
ten adapted parameters --- converging from initial
offsets within $\pm10\%$ of true values to below
$2\%$ within $5$~[s] for inertia/stiffness and
$10$~[s] for density, with $\pm20\%$ projection
bounds enforced throughout;
(b)~minimum eigenvalue of the windowed regressor
Gramian $\int_t^{t+T}\bar{\mathbf{Y}}_i^\top
\bar{\mathbf{Y}}_i\,d\tau$ for both links,
confirming PE condition~(\ref*{eq:PE}) throughout
the active adaptation phase.}
\label{fig:adaptation}
\end{figure*}

\begin{figure*}[!t]
\centering
\subfloat[]{
\includegraphics[width=0.48\textwidth]{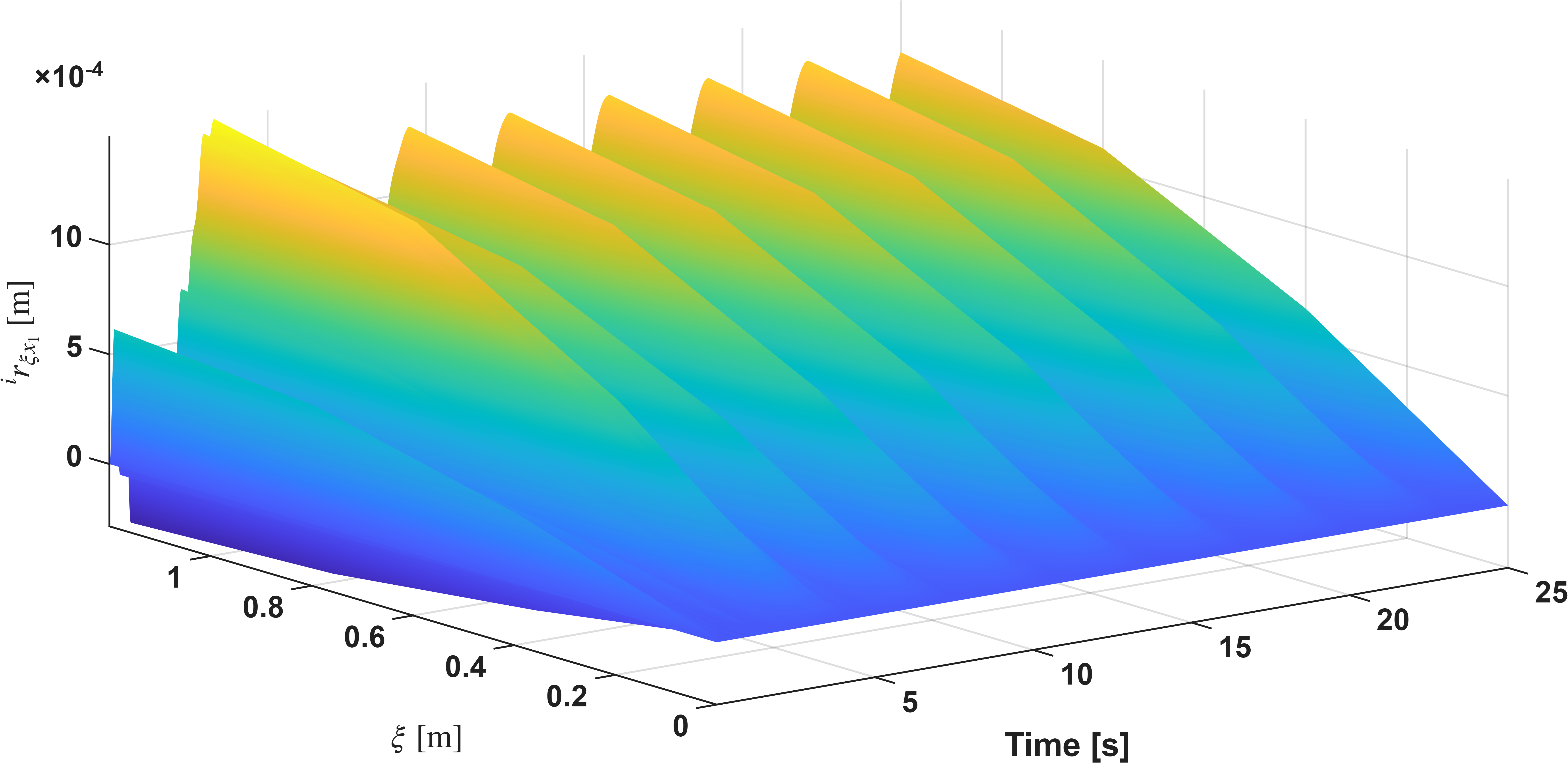}
\label{fig:def_x1}}
\hfill
\subfloat[]{
\includegraphics[width=0.48\textwidth]{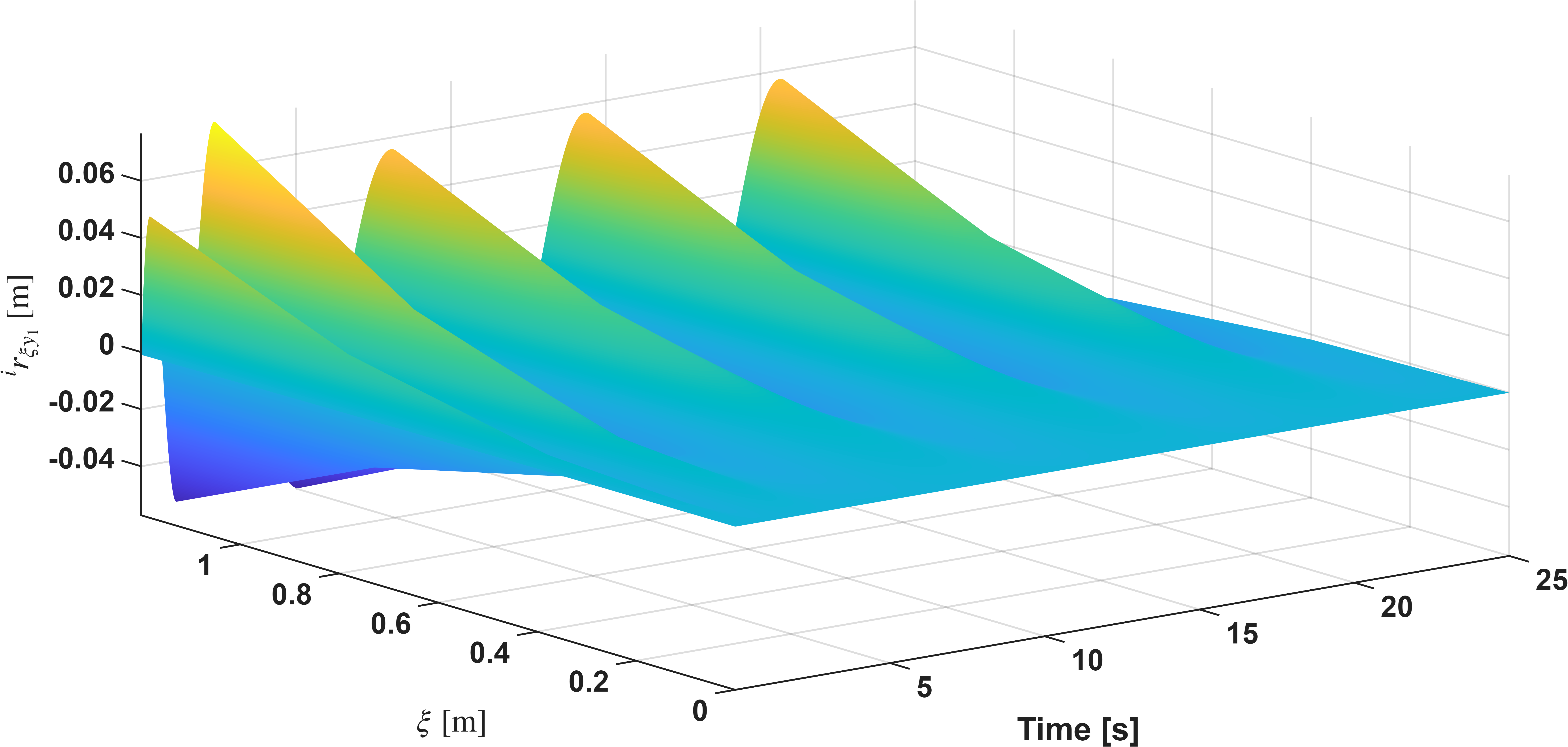}
\label{fig:def_y1}}

\vspace{0.5em}

\subfloat[]{
\includegraphics[width=0.48\textwidth]{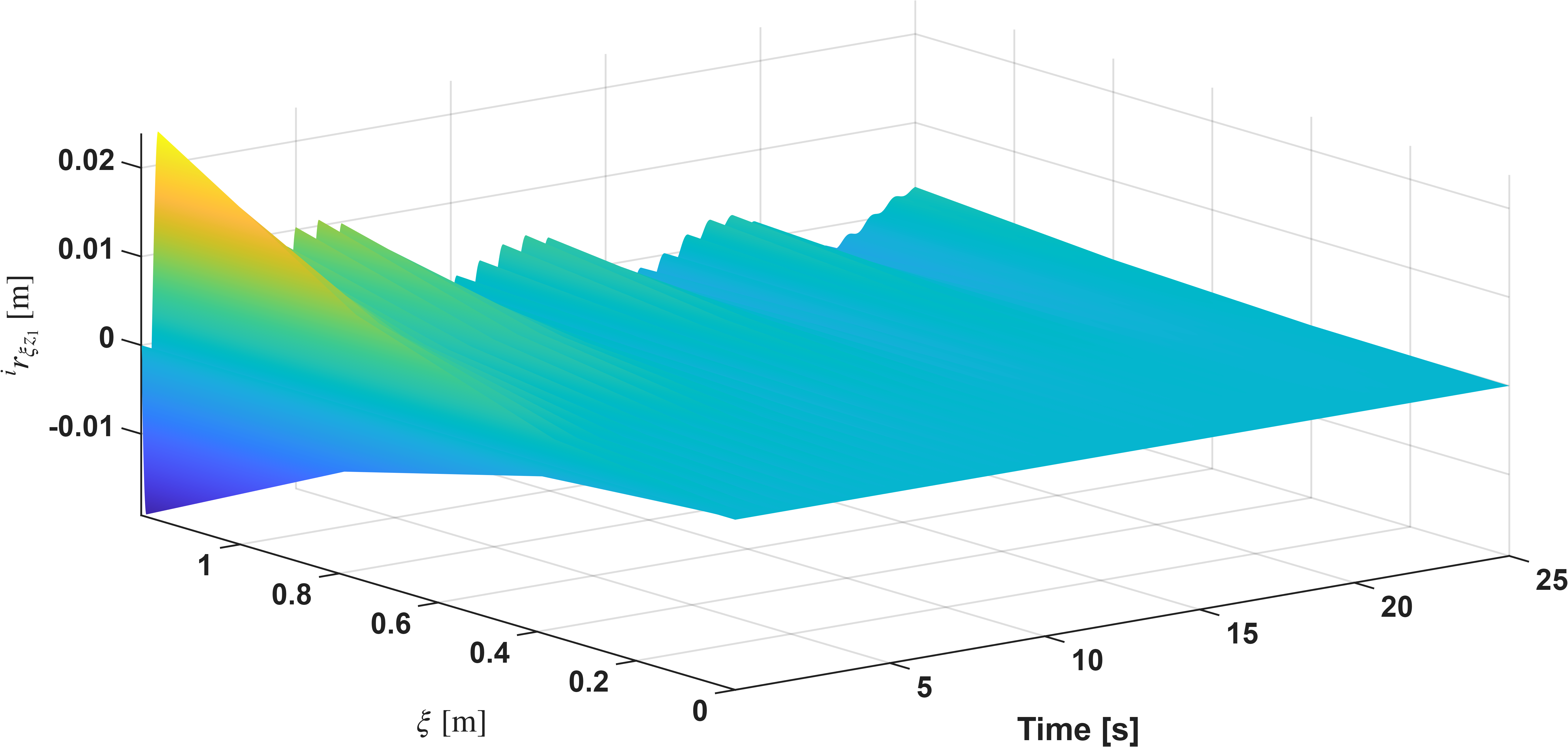}
\label{fig:def_z1}}
\hfill
\subfloat[]{
\includegraphics[width=0.48\textwidth]{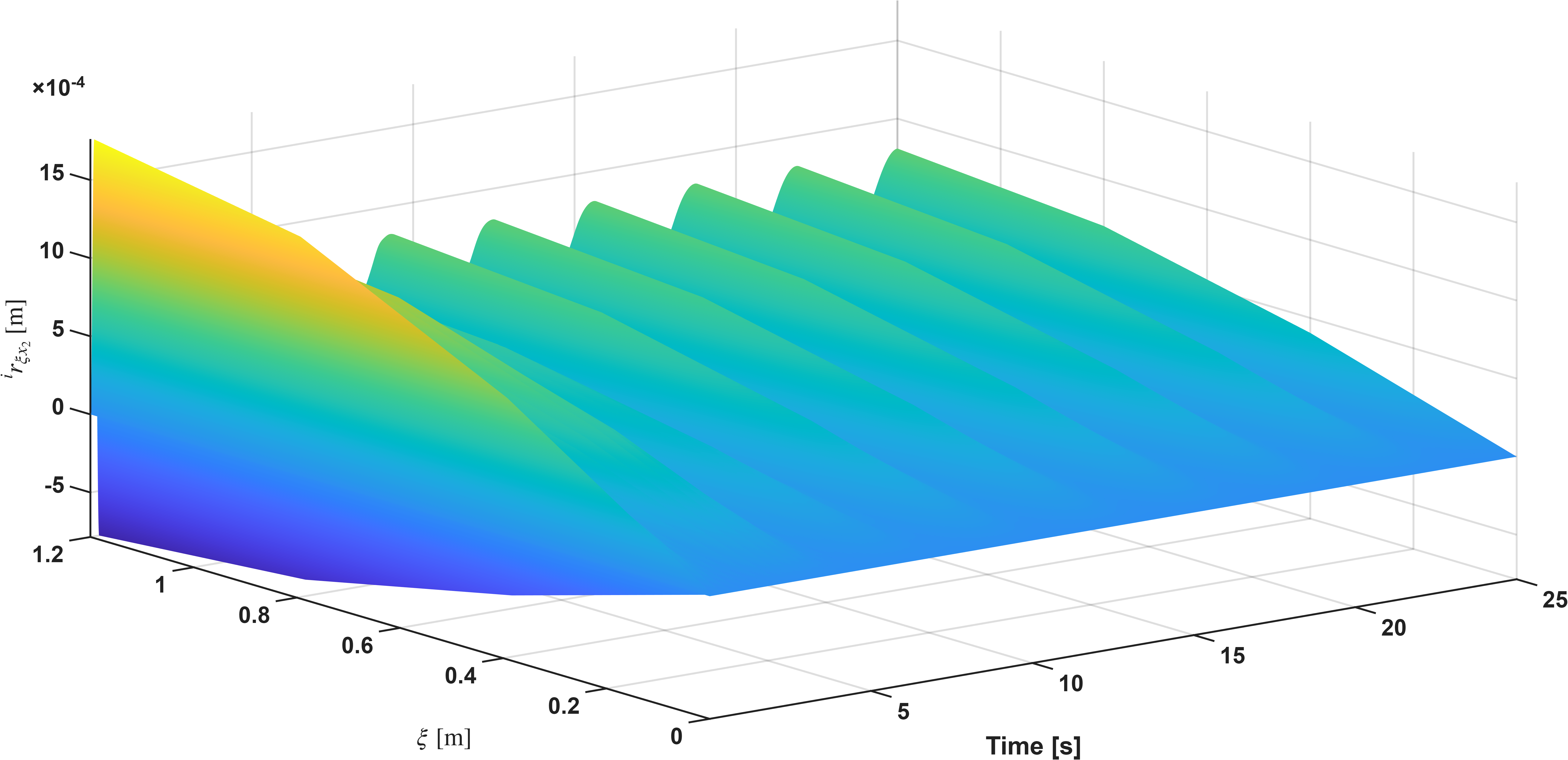}
\label{fig:def_x2}}

\vspace{0.5em}

\subfloat[]{
\includegraphics[width=0.48\textwidth]{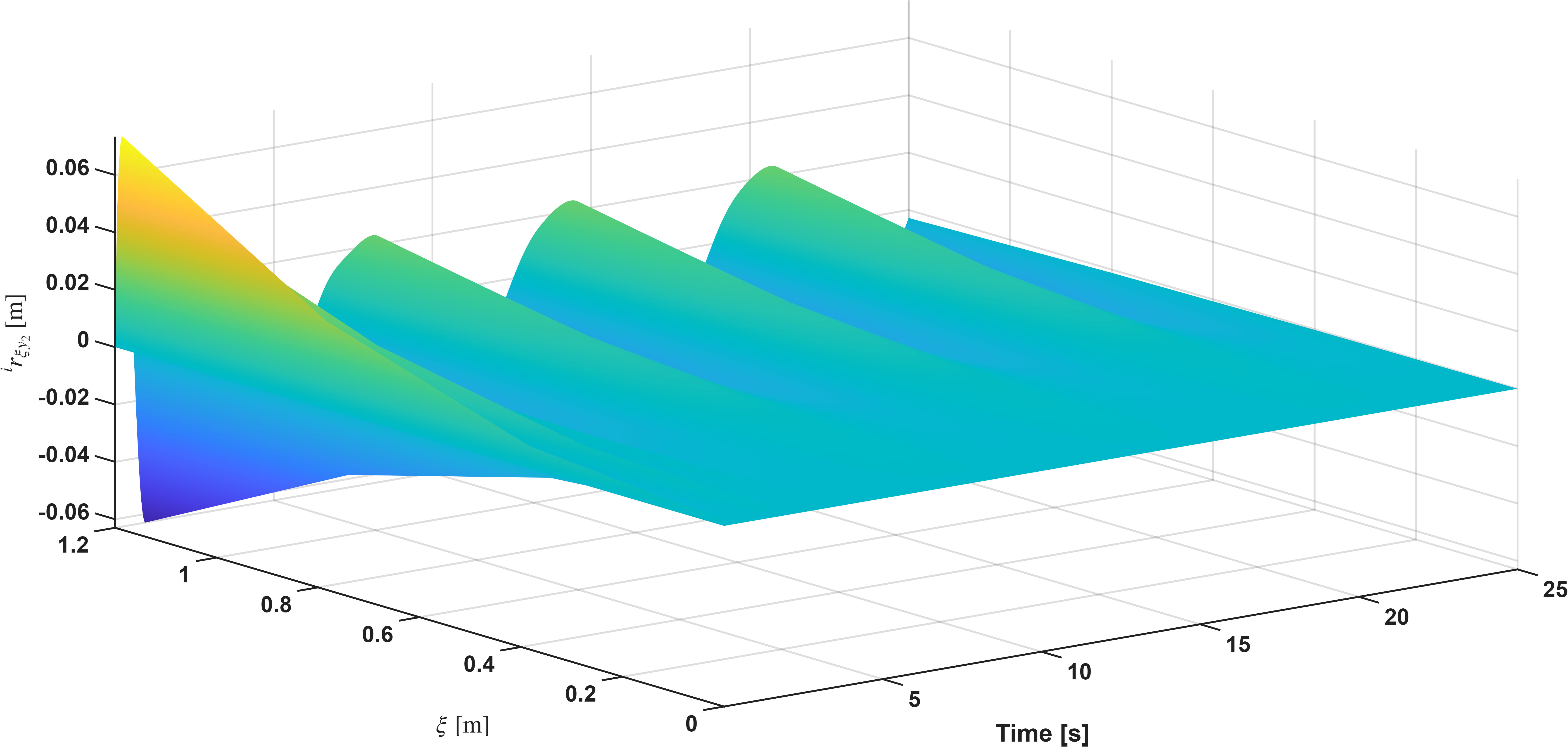}
\label{fig:def_y2}}
\hfill
\subfloat[]{
\includegraphics[width=0.48\textwidth]{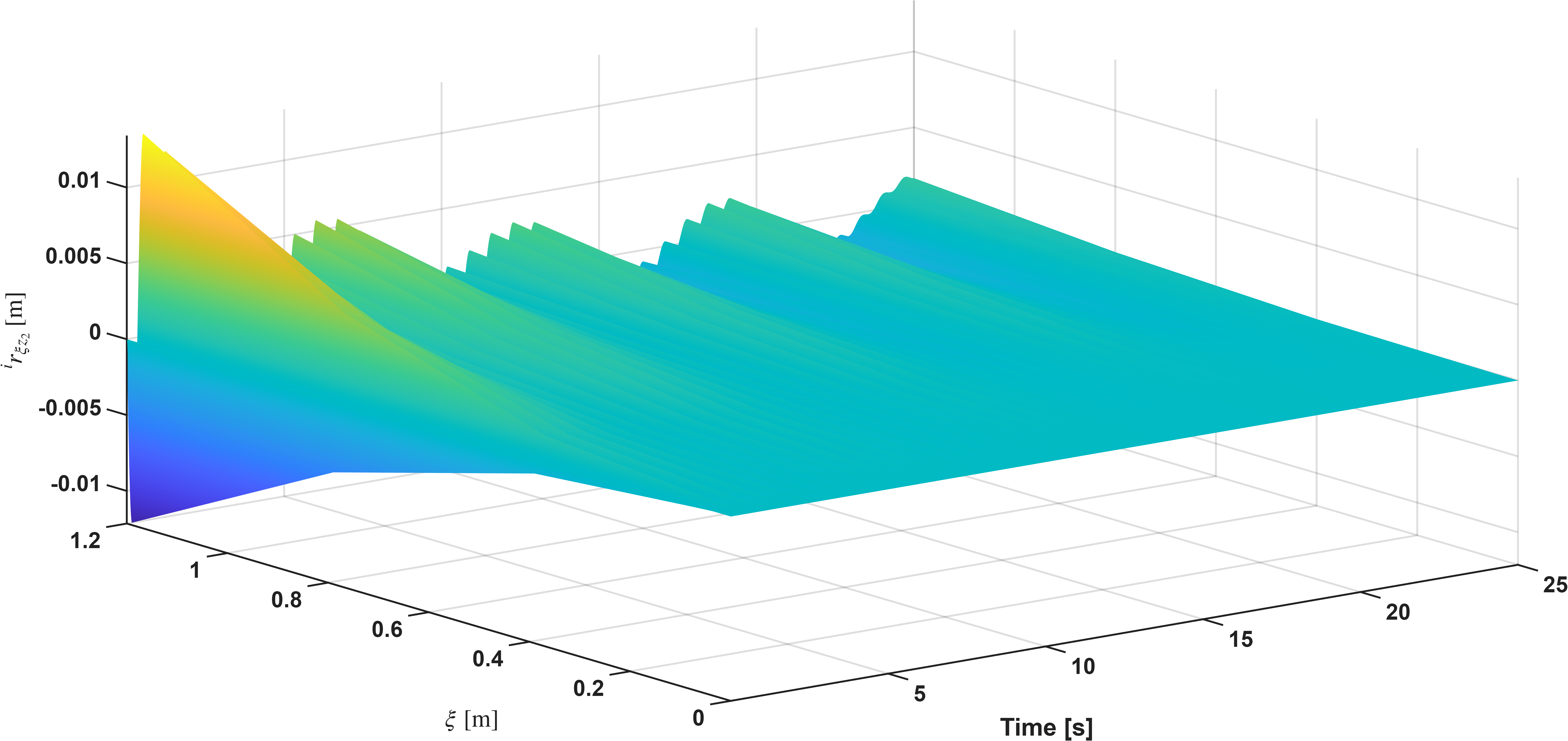}
\label{fig:def_z2}}

\caption{Distributed deformation field of both links:
(a)~${}^i r_{\xi x_1}$ (axial, order $10^{-4}$~[m]),
(b)~${}^i r_{\xi y_1}$ (bending),
(c)~${}^i r_{\xi z_1}$ (bending),
(d)~${}^i r_{\xi x_2}$ (axial, order $10^{-4}$~[m]),
(e)~${}^i r_{\xi y_2}$ (bending),
(f)~${}^i r_{\xi z_2}$ (bending).
Deformation satisfies clamped-free boundary
conditions~(\ref*{eq:35})--(\ref*{eq:38}) with zero
at the base and maximum at the free tip for both
links throughout the simulation horizon.}
\label{fig:deformation}
\end{figure*}

\section{Conclusion} \label{sec:conclusion}

A subsystem-based adaptive control framework for serial
flexible manipulators with an arbitrary number of links
has been presented, grounded in a screw-theoretic
Lie-algebraic model in which all dynamic quantities are
expressed as body-fixed twists and dual wrenches on
$\mathfrak{se}(3)$. The per-link dynamics were cast in
controllable form by substituting the strain-based
deformation PDE into the dynamic equation, eliminating
the distributed elastic acceleration term. A nominal
subsystem-level controller was proven to produce
exponential twist error decay via a per-subsystem
Lyapunov function $\nu_i$; the interaction power terms
at the subsystem boundaries telescope to zero upon
summation over all links by Newton's third law and the
frame invariance of the natural power pairing on
$\mathfrak{se}(3)\times\mathfrak{se}^*(3)$, establishing
exponential stability of the full $n$-link system without
cross-term analysis. An adaptive modification replaces
exact parameters with online estimates governed by a
projection-based law, with exponential convergence of
both twist error and parameter estimation error
established at subsystem and system levels. The screw-theoretic structure renders interaction term
cancellation exact, the stability certificate modular,
and the synthesis automatable for arbitrary $n$.

Several directions remain open. On the sensing side,
integration of IMU networks along the links offers a
field-deployable means of reconstructing the distributed
elastic state required by the deflection-compensating
reference generation, without recourse to
laboratory-grade optical or strain-gauge sensing.
Incorporation of external disturbances --- fluid loading,
payload uncertainty, base vibration --- into the
stability analysis requires either disturbance-observer
augmentation or robust Lyapunov modifications.
The subsystem architecture is compatible with
alternative controller structures beyond the
twist-error proportional law adopted here; passivity-based,
model predictive, and sliding-mode designs within the
same $\mathfrak{se}(3)$ decomposition are natural
continuations offering improved transient performance
or stronger robustness for specific application classes.

\bibliographystyle{IEEEtran}
\bibliography{references} 

\end{document}